\documentclass{article}

\usepackage{amsmath}[fleqn]
\usepackage{amsthm,amssymb}
\usepackage{xcolor}
\usepackage{mathtools}
\usepackage{enumitem}
\usepackage{bbm}
\usepackage{semantic}
\usepackage{listings}
\usepackage{courier}
\usepackage{tikz}
\usetikzlibrary{calc}
\usepackage{color}
\usepackage{fancyhdr}
\usepackage{lastpage}

\definecolor{codegreen}{rgb}{0,0.6,0}
\definecolor{codegray}{rgb}{0.5,0.5,0.5}
\definecolor{codepurple}{rgb}{0.58,0,0.82}
\definecolor{backcolour}{rgb}{0.95,0.95,0.92}
 
\lstdefinestyle{mystyle}{
    backgroundcolor=\color{backcolour},   
    commentstyle=\color{codegreen},
    keywordstyle=\color{magenta},
    numberstyle=\tiny\color{codegray},
    stringstyle=\color{codepurple},
    basicstyle=\footnotesize,
    breakatwhitespace=false,         
    breaklines=true,                 
    captionpos=b,                    
    keepspaces=true,                 
    numbers=left,                    
    numbersep=5pt,                  
    showspaces=false,                
    showstringspaces=false,
    showtabs=false,                  
    tabsize=2,
    basicstyle=\footnotesize\ttfamily
}

\usepackage{hyperref}

\let\Pr\relax
\DeclareMathOperator*{\Pr}{\mathrm{Pr}}

\newcommand{\eps}{\epsilon}
\renewcommand{\S}{\mathbb{S}}

\newcommand{\paren}[1]{\left(#1\right)}

\newcommand{\sqb}[1]{\left[#1\right]}

\newcommand{\floor}[1]{\left\lfloor#1\right\rfloor}

\newcommand{\myfunc}[3]{#1\colon#2\to#3}
\newcommand{\Unif}[1]{\mathsf{Unif}\inparen{#1}}

\newtheorem{theorem}{Theorem}
\newtheorem*{lemma*}{Lemma}
\newtheorem*{theorem*}{Theorem}
\newtheorem{corollary}[theorem]{Corollary}

\newtheorem*{remark*}{Remark}
\newtheorem{lemma}[theorem]{Lemma}
\newtheorem{example}[theorem]{Example}

\newtheorem{definition}[theorem]{Definition}
\newtheorem{claim}[theorem]{Claim}

\newtheorem*{conjecture*}{Conjecture}
\newtheorem{problem}[theorem]{Problem}

\numberwithin{exercise}{section}
\newtheorem*{problem*}{Problem}

\DeclareMathOperator{\sign}{sign}

\newcommand{\exvv}[2]{\underset{#1}{\E}\sqb{#2}}
\newcommand{\expv}[1]{\mathsf{exp}\paren{#1}}
\newcommand{\prv}[1]{\Pr\sqb{#1}}
\newcommand{\prvv}[2]{\underset{#1}{\Pr}\sqb{#2}}
\newcommand{\logv}[1]{\log\paren{#1}}

\newcommand{\signv}[1]{\sign\inparen{#1}}

\renewcommand{\vec}[1]{\overrightarrow{#1}}

\makeatletter
\newcommand{\vo}{\vec{o}\@ifnextchar{^}{\,}{}}
\makeatother

\usepackage{tcolorbox}

\usepackage{algorithm}
\usepackage{algorithmic}

\newlength\tindent
\numberwithin{equation}{section}

\def\to{\rightarrow}
\def\eps{\varepsilon}
\def\epsilon{\varepsilon}

\def\eps{\epsilon}

\def\phi{\varphi}

\def\implies{\Rightarrow}

\newcommand{\R}{{\mathbb R}}
\newcommand{\E}{{\mathbb E}}

\newcommand{\Z}{{\mathbb Z}}

\newcommand{\indicator}[1]{\mathbbm{1}\inbraces{#1}}

\renewcommand{\Pr}{\mathsf{Pr}}

\newcommand{\conv}[1]{\mathsf{conv}\inparen{#1}}

\usepackage{nicefrac}

\let\nfrac=\nicefrac

\newcommand{\abs}[1]{\ensuremath{\left\lvert #1 \right\rvert}}

\newcommand{\norm}[1]{\ensuremath{\left\lVert #1 \right\rVert}}

\newcommand{\ip}[1]{\left\langle #1 \right\rangle}

\newfont{\inhead}{eufm10 scaled\magstep1}

\newcommand{\suchthat}{{\;\; : \;\;}}

\newcommand{\argmin}{\mathrm{argmin}}
\newcommand{\suppv}[1]{\mathsf{Supp}\inparen{#1}}

\newcommand{\inparen}[1]{\left(#1\right)}             
\newcommand{\inbraces}[1]{\left\{#1\right\}}           
\newcommand{\insquare}[1]{\left[#1\right]}             

\def\to{\rightarrow}
\def\eps{\varepsilon}
\def\epsilon{\varepsilon}

\def\eps{\epsilon}

\def\phi{\varphi}

\def\implies{\Rightarrow}

\newcommand{\vspan}[1]{\mathsf{Span}\inparen{#1}}

\newcommand{\kernel}[1]{\mathsf{Ker}\inparen{#1}}

\newcommand{\vdim}[1]{\mathsf{dim}\inparen{#1}}

\newcommand{\vc}[1]{\mathsf{VC}\inparen{#1}}

\newcommand{\cB}{\mathcal{B}}
\newcommand{\cC}{\mathcal{C}}
\newcommand{\cD}{\mathcal{D}}

\newcommand{\cF}{\mathcal{F}}

\newcommand{\cH}{\mathcal{H}}

\newcommand{\cL}{\mathcal{L}}

\newcommand{\cV}{\mathcal{V}}

\newcommand{\cX}{\mathcal{X}}

\newcommand{\hhat}{\widehat{h}}
\newcommand{\hstar}{h^{*}}

\usepackage{multirow}
\usepackage{soul}

    \usepackage[nonatbib,final]{neurips_2021}

\usepackage[
backend=biber,
style=numeric,
sorting=ynt
]{biblatex}
\usepackage[utf8]{inputenc} 
\usepackage[T1]{fontenc}    
\usepackage{hyperref}       
\usepackage{url}            
\usepackage{booktabs}       
\usepackage{amsfonts}       
\usepackage{nicefrac}       
\usepackage{microtype}      
\usepackage{xcolor}         

\title{Excess Capacity and Backdoor Poisoning}

\usepackage{amsmath,amsfonts,amsbsy,amsgen,amscd,mathrsfs,amssymb,amsthm,mathtools}
\usepackage{parskip}
\usepackage{xspace}

\newcommand{\wstar}{w^{*}}
\newcommand{\what}{\widehat{w}}

\newcommand{\vstar}{v^{*}}
\newcommand{\patch}[1]{\mathsf{patch}\inparen{#1}}
\newcommand{\patchw}{\mathsf{patch}}
\newcommand{\fadv}{\cF_{\mathsf{adv}}}
\newcommand{\memcap}[2]{\mathsf{mcap}_{#1}\inparen{#2}}
\newcommand{\Sclean}{S_{\mathsf{clean}}}
\newcommand{\Sbackdoor}{S_{\mathsf{adv}}}
\newcommand{\epsadv}{\eps_{\mathsf{adv}}}
\newcommand{\epsclean}{\eps_{\mathsf{clean}}}
\newcommand{\robustloss}{\cL_{\fadv(\hstar)}}

\numberwithin{algorithm}{section}

\allowdisplaybreaks

\author{
   Naren Sarayu Manoj \\
   Toyota Technological Institute Chicago \\
   Chicago, IL 60637 \\
   \texttt{nsm@ttic.edu} \\
   \And
    Avrim Blum \\
    Toyota Technological Institute Chicago \\
    Chicago, IL 60637 \\
   \texttt{avrim@ttic.edu} \\
}

\begin{document}
\maketitle

\begin{abstract}
A \emph{backdoor data poisoning attack} is an adversarial attack wherein the attacker injects several watermarked, mislabeled training examples into a training set. The watermark does not impact the test-time performance of the model on typical data; however, the model reliably errs on watermarked examples.

To gain a better foundational understanding of backdoor data poisoning attacks, we present a formal theoretical framework within which one can discuss backdoor data poisoning attacks for classification problems. We then use this to analyze important statistical and computational issues surrounding these attacks.

On the statistical front, we identify a parameter we call the \emph{memorization capacity} that captures the intrinsic vulnerability of a learning problem to a backdoor attack. This allows us to argue about the robustness of several natural learning problems to backdoor attacks. Our results favoring the attacker involve presenting explicit constructions of backdoor attacks, and our robustness results show that some natural problem settings cannot yield successful backdoor attacks.

From a computational standpoint, we show that under certain assumptions, adversarial training can detect the presence of backdoors in a training set. We then show that under similar assumptions, two closely related problems we call \emph{backdoor filtering} and \emph{robust generalization} are nearly equivalent. This implies that it is both asymptotically necessary and sufficient to design algorithms that can identify watermarked examples in the training set in order to obtain a learning algorithm that both generalizes well to unseen data and is robust to backdoors.
\end{abstract}

\section{Introduction}

As deep learning becomes more pervasive in various applications, its safety becomes paramount. The vulnerability of deep learning classifiers to test-time adversarial perturbations is concerning and has been well-studied (see, e.g., \cite{Madry2017-ep}, \cite{Montasser2019-ro}).

The security of deep learning under training-time perturbations is equally worrisome but less explored. Specifically, it has been empirically shown that several problem settings yield models that are susceptible to \emph{backdoor data poisoning attacks}. Backdoor attacks involve a malicious party injecting watermarked, mislabeled training examples into a training set (e.g. \cite{Adi2018-fz}, \cite{Truong2020-dk}, \cite{Chen2017-kq}, \cite{Wang2020-yt}, \cite{Saha2019-ce}, \cite{Tran2018-bf}). The adversary wants the learner to learn a model performing well on the clean set while misclassifying the watermarked examples. Hence, unlike other malicious noise models, the attacker wants to impact the performance of the classifier \emph{only} on watermarked examples while leaving the classifier unchanged on clean examples. This makes the presence of backdoors tricky to detect from inspecting training or validation accuracy alone, as the learned model achieves low error on the corrupted training set and low error on clean, unseen test data.

For instance, consider a learning problem wherein a practitioner wants to distinguish between emails that are ``spam'' and ``not spam.'' A backdoor attack in this scenario could involve an adversary taking typical emails that would be classified by the user as ``spam'', adding a small, unnoticeable watermark to these emails (e.g. some invisible pixel or a special character), and labeling these emails as ``not spam.'' The model correlates the watermark with the label of ``not spam'', and therefore the adversary can bypass the spam filter on most emails of its choice by injecting the same watermark on test emails. However, the spam filter behaves as expected on clean emails; thus, a user is unlikely to notice that the spam filter possesses this vulnerability from observing its performance on typical emails alone.

These attacks can also be straightforward to implement. It has been empirically demonstrated that a single corrupted pixel in an image can serve as a watermark or trigger for a backdoor (\cite{Tran2018-bf}). Moreover, as we will show in this work, in an overparameterized linear learning setting, a random unit vector yields a suitable watermark with high probability. Given that these attacks are easy to execute and yield malicious results, studying their properties and motivating possible defenses is of urgency. Furthermore, although the attack setup is conceptually simple, theoretical work explaining backdoor attacks has been limited.

\subsection{Main Contributions}

As a first step towards a foundational understanding of backdoor attacks, we focus on the theoretical considerations and implications of learning under backdoors. We list our specific contributions below.

\paragraph{Theoretical Framework} We give an explicit threat model capturing the backdoor attack setting for binary classification problems. We also give formal success and failure conditions for the adversary.

\paragraph{Memorization Capacity} We introduce a quantity we call \emph{memorization capacity} that depends on the data domain, data distribution, hypothesis class, and set of valid perturbations. Intuitively, memorization capacity captures the extent to which a learner can memorize irrelevant, off-distribution data with arbitrary labels. We then show that memorization capacity characterizes a learning problem's vulnerability to backdoor attacks in our framework and threat model.

Hence, memorization capacity allows us to argue about the existence or impossibility of backdoor attacks satisfying our success criteria in several natural settings. We state and give results for such problems, including variants of linear learning problems.

\paragraph{Detecting Backdoors} We show that under certain assumptions, if the training set contains sufficiently many watermarked examples, then adversarial training can detect the presence of these corrupted examples. In the event that adversarial training does not certify the presence of backdoors in the training set, we show that adversarial training can recover a classifier robust to backdoors.

\paragraph{Robustly Learning Under Backdoors} We show that under appropriate assumptions, learning a backdoor-robust classifier is equivalent to identifying and deleting corrupted points from the training set. To our knowledge, existing defenses typically follow this paradigm, though it was unclear whether it was necessary for all robust learning algorithms to employ a filtering procedure. Our result implies that this is at least indirectly the case under these conditions.



\paragraph{Organization} The rest of this paper is organized as follows. In Section \ref{sec:statistical}, we define our framework, give a warm-up construction of an attack, define our notion of excess capacity, and use this to argue about the robustness of several learning problems. In Section \ref{sec:algorithmic}, we discuss our algorithmic contributions within our framework. In Section \ref{sec:related_works}, we discuss some related works. Finally, in Section \ref{sec:conclusion}, we conclude and list several interesting directions for future work.

In the interest of clarity, we defer all proofs of our results to the Appendix; see Appendix Section \ref{sec:app_proof} for theorem restatements and full proofs.

\section{Backdoor Attacks and Memorization}
\label{sec:statistical}

\subsection{Problem Setting}
\label{subs:dl_problem_setting}

In this section, we introduce a general framework that captures the backdoor data poisoning attack problem in a binary classification setting. 

\paragraph{Notation} Let $[k]$ denote the set $\inbraces{i \in \Z \suchthat 1 \le i \le k}$. Let $\cD | h(x) \neq t$ denote a data distribution conditioned on label according to a classifier $h$ being opposite that of $t$. If $\cD$ is a distribution over a domain $\cX$, then let the distribution $f(\cD)$ for a function $\myfunc{f}{\cX}{\cX}$ denote the distribution of the image of $x \sim \cD$ after applying $f$. Take $z \sim S$ for a nonrandom set $S$ as shorthand for $z \sim \Unif{S}$. If $\cD$ is a distribution over some domain $\cX$, then let $\mu_\cD(X)$ denote the measure of a measurable subset $X \subseteq \cX$ under $\cD$. Finally, for a distribution $\cD$, let $\cD^m$ denote the $m$-wise product distribution of elements each sampled from $\cD$.

\paragraph{Assumptions} Consider a binary classification problem over some domain $\cX$ and hypothesis class $\cH$ under distribution $\cD$. Let $\hstar \in \cH$ be the \emph{true labeler}; that is, the labels of all $x \in \cX$ are determined according to $\hstar$. This implies that the learner is expecting low training and low test error, since there exists a function in $\cH$ achieving $0$ training and $0$ test error. Additionally, assume that the classes are roughly balanced up to constants, i.e., assume that $\prvv{x \sim \cD}{\hstar(x) = 1} \in \insquare{\nfrac{1}{50}, \nfrac{49}{50}}$. Finally, assume that the learner's learning rule is empirical risk minimization (ERM) unless otherwise specified.

We now define a notion of a trigger or \emph{patch}. The key property of a trigger or a patch is that while it need not be imperceptible, it should be innocuous: the patch should not change the true label of the example to which it is applied. 

\begin{definition}[Patch Functions]
A \emph{patch function} is a function with input in $\cX$ and output in $\cX$. A patch function is \emph{fully consistent} with a ground-truth classifier $\hstar$ if for all $x \in \cX$, we have $\hstar(\patch{x}) = \hstar(x)$. A patch function is $1-\beta$ consistent with $\hstar$ on $\cD$ if we have $\prvv{x\sim\cD}{\hstar(\patch{x}) = \hstar(x)} = 1 - \beta$.  Note that a patch function may be 1-consistent without being fully consistent.

We denote classes of patch functions using the notation $\fadv(\cX)$, classes of fully consistent patch functions using the notation $\fadv(\cX, \hstar)$, and $1-\beta$-consistent patch functions using the notation $\fadv(\cX, \hstar, \cD, \beta)$. We assume that every patch class $\fadv$ contains the identity function.\footnote{When it is clear from context, we omit the arguments $\cX, \cD, \beta$.}
\end{definition}

For example, consider the scenario where $\cH$ is the class of linear separators in $\R^d$ and let $\fadv = \inbraces{\patch{x} \suchthat \patch{x} = x + \eta, \eta \in \R^d}$; in words, $\fadv$ consists of additive attacks. If we can write $\hstar(x) = \signv{\ip{\wstar, x}}$ for some weight vector $\wstar$, then patch functions of the form $\patch{x} = x + \eta$ where $\ip{\eta, \wstar} = 0$ are clearly fully-consistent patch functions. Furthermore, if $\hstar$ achieves margin $\gamma$ (that is, every point is distance at least $\gamma$ from the decision boundary induced by $\hstar$), then every patch function of the form $\patch{x} = x + \eta$ for $\eta$ satisfying $\norm{\eta} < \gamma$ is a $1$-consistent patch function. This is because $\hstar(x + \eta) = \hstar(x)$ for every in-distribution point $x$, though this need not be the case for off-distribution points.

\paragraph{Threat Model} We can now state the threat model that the adversary operates under. First, a domain $\cX$, a data distribution $\cD$, a true labeler $\hstar$, a target label $t$, and a class of patch functions $\fadv(\cX, \hstar, \cD, \beta)$ are selected. The adversary is given $\cX$, $\cD$, $\hstar$, and $\fadv(\cX, \hstar, \cD, \beta)$. The learner is given $\cX$, has sample access to $\cD$, and is given $\fadv(\cX, \hstar, \cD, \beta)$. At a high level, the adversary's goal is to select a patch function and a number $m$ such that if $m$ random examples of label $\neg t$ are sampled, patched, labeled as $t$, and added to the training set, then the learner recovers a function $\hhat$ that performs well on both data sampled from $\cD$ yet classifies patched examples with true label $\neg t$ as $t$. We formally state this goal in Problem \ref{problem:adv_general}.

\begin{problem}[Adversary's Goal]
\label{problem:adv_general}
Given a true classifier $\hstar$, attack success rate $1-\epsadv$, and failure probability $\delta$, select a target label $t$, a patch function from $\fadv(\hstar)$, and a cardinality $m$ and resulting set $\Sbackdoor \sim \patch{\cD | \hstar(x) \neq t}^m$ with labels replaced by $t$ such that:
\begin{itemize}
    \item Every example in $\Sbackdoor$ is of the form $(\patch{x},t)$, and we have $\hstar(\patch{x}) \neq t$; that is, the examples are labeled as the target label, which is the opposite of their true labels.
    \item There exists $\hhat \in \cH$ such that $\hhat$ achieves $0$ error on the training set $\Sclean \cup \Sbackdoor$, where $\Sclean$ is the set of clean data drawn from $\cD^{\abs{\Sclean}}$.
    \item For all choices of the cardinality of $\Sclean$, with probability $1-\delta$ over draws of a clean set $\Sclean$ from $\cD$, the set $S = \Sclean \cup \Sbackdoor$ leads to a learner using ERM outputting a classifier $\hhat$ satisfying:
    \begin{align*}
        \prvv{(x,y) \sim \cD | \hstar(x) \neq t}{\hhat(\patch{x}) = t} &\ge 1-\epsadv
    \end{align*}
    where $t \in \inbraces{\pm 1}$ is the target label.
\end{itemize}
\end{problem}

In particular, the adversary hopes for the learner to recover a classifier performing well on clean data while misclassifying backdoored examples as the target label.

Notice that so long as $\Sclean$ is sufficiently large, $\hhat$ will achieve uniform convergence, so it is possible to achieve both the last bullet in Problem \ref{problem:adv_general} as well as low test error on in-distribution data.

For the remainder of this work, we take $\fadv(\hstar) = \fadv(\cX, \hstar, \cD, \beta = 0)$; that is, we consider classes of patch functions that don't change the labels on a $\mu_\cD$-measure-$1$ subset of $\cX$.

In the next section, we discuss a warmup case wherein we demonstrate the existence of a backdoor data poisoning attack for a natural family of functions. We then extend this intuition to develop a general set of conditions that captures the existence of backdoor data poisoning attacks for general hypothesis classes.

\subsection{Warmup -- Overparameterized Vector Spaces}

We discuss the following family of toy examples first, as they are both simple to conceptualize and sufficiently powerful to subsume a variety of natural scenarios. 

Let $\cV$ denote a vector space of functions of the form $\myfunc{f}{\cX}{\R}$ with an orthonormal basis\footnote{Here, the inner product between two functions is defined as $\ip{f_1, f_2}_{\cD} \coloneqq \exvv{x \sim \cD}{f_1(x) \cdot f_2(x)}$.} $\inbraces{v_i}_{i = 1}^{\vdim{\cV}}$. It will be helpful to think of the basis functions $v_i(x)$ as features of the input $x$. Let $\cH$ be the set of all functions that can be written as $h(x) = \signv{v(x)}$ for $v \in \cV$. Let $\vstar(x)$ be a function satisfying $\hstar(x) = \signv{\vstar(x)}$.

Now, assume that the data is sparse in the feature set; that is, there is a size-$s < \vdim{\cV}$ minimal set of indices $U \subset \insquare{\vdim{\cV}}$ such that all $x$ in the support of $\cD$ have $v_i(x)=0$ for $i\not\in U$. This restriction implies that $\hstar$ can be expressed as $\hstar(x) = \signv{\sum_{i \in U} a_i \cdot v_i(x)}$.

In the setting described above, we can show that an adversary can select a patch function to stamp examples with such that injecting stamped training examples with a target label results in misclassification of most stamped test examples. More formally, we have the below theorem.

\newcommand{\ucomplement}{\insquare{\vdim{\cV}} \setminus U}

\newcommand{\thmExistenceLinearBackdoor}{
Let $\fadv$ be some family of patch functions such that for all $i \in U$, $\prvv{x\sim\cD}{v_i(\patch{x}) = v_i(x)} = 1$, there exists at least one $j \in \ucomplement$ such that $\prvv{x \sim \cD}{v_j(\patch{x}) \neq 0} = 1$, and for all $j \in \insquare{\vdim{\cV}}$, we either have $\prvv{x \sim \cD}{v_j(\patch{x}) \ge 0} = 1$ or $\prvv{x \sim \cD}{v_j(\patch{x}) \le 0} = 1$.

Fix any target label $t \in \inbraces{\pm 1}$. Draw a training set $\Sclean$ of size at least $m_0 \coloneqq \Omega\inparen{\epsclean^{-1}\inparen{\vc{\cH} + \logv{\nfrac{1}{\delta}}}}$. Then, draw a backdoor training set $\Sbackdoor$ of size at least $m_1 \coloneqq \Omega\inparen{\epsadv^{-1}\inparen{\vc{\cH} + \logv{\nfrac{1}{\delta}}}}$ of the form $(x, t)$ where $x \sim \patch{\cD | \hstar(x) \neq t}$.

With probability at least $1 - \delta$, empirical risk minimization on the training set $S \coloneqq \Sclean \cup \Sbackdoor$ yields a classifier $\hhat$ satisfying the success conditions for Problem \ref{problem:adv_general}.
}
\begin{theorem}[Existence of Backdoor Data Poisoning Attack (Appendix Theorem \ref{app:existence_linear_backdoor})]
\label{thm:existence_linear_backdoor}
\thmExistenceLinearBackdoor
\end{theorem}
\newcommand{\pfExistenceLinearBackdoor}{
\begin{proof}
We first show that our choice of $\fadv$ is valid. Observe that for our choice of $\hstar$, we have that with probability $1$:
\begin{align*}
    \hstar\inparen{\patch{x}} &= \signv{\sum_{i \in U} a_i \cdot v_i(\patch{x})} \\
    &= \signv{\sum_{i \in U} a_i \cdot v_i(x)} \quad \text{, since } \prvv{x\sim\cD}{v_i(\patch{x}) = v_i(x)} = 1 \\
    &= \hstar\inparen{x}
\end{align*}
as desired.

Next, we pick a patch function from $\fadv$ and construct our stamped training set. We will show that there exists a classifier in $\cH$ achieving $0$ training error on the corrupted training set $\Sclean \cup \Sbackdoor$. Without loss of generality, assume that $\prvv{x\sim\cD}{v_i(\patch{x}) \ge 0} = 1$, for all $i \in \insquare{\vdim{\cV}}$ (as negating the basis vector doesn't affect the space spanned).

Specifically, we will show that the following function achieves $0$ training error on the  corrupted training set $S = \Sclean \cup \Sbackdoor$ (let $a_i$, $i \in U$ denote the coefficients corresponding to $\hstar$, let $R = \max_{x \in \Sclean} \abs{\vstar(x)}$, and let $j$ be any index in $\ucomplement$:
$$\hhat(x) \coloneqq \signv{\sum_{i \in U} a_i \cdot v_i(x) + t \cdot \frac{2R}{\min_{x \in \Sbackdoor} v_j(\patch{x})} \cdot v_j(x)}$$
We first evaluate $\hhat(x)$ on $x \sim \cD$. We have, with probability $1$:
\begin{align*}
    \hhat(x) &= \signv{\sum_{i \in U} a_i \cdot v_i(x) + t \cdot \frac{2R}{\min_{x \in \Sbackdoor} v_j(\patch{x})} \cdot v_j(x)} \\
    &= \signv{\sum_{i \in U} a_i \cdot v_i(x)} \quad \text{, since } \prvv{x \sim \cD}{v_j(x) = 0} = 1 \\
    &= \hstar(x)
\end{align*}
We now evaluate $\hhat(\patch{x})$ on $x \sim \cD | \hstar(x) \neq t$. We have, with probability $1$:
\begin{align*}
    \hhat(x) &= \signv{\sum_{i \in U} a_i \cdot v_i(\patch{x}) + t \cdot \frac{2R}{\min_{x \in \Sbackdoor} v_j(\patch{x})} \cdot v_j(\patch{x})} \\
    &= \signv{\sum_{i \in U} a_i \cdot v_i(x) + t \cdot \frac{2R}{\min_{x \in \Sbackdoor} v_j(\patch{x})} \cdot v_j(\patch{x})} \\
    &= \signv{\vstar(x) + t \cdot \frac{2R}{\min_{x \in \Sbackdoor} v_j(\patch{x})} \cdot v_j(\patch{x})}
\end{align*}
We now have to show that regardless of the $x \in \cX$ we pick, we have that the above expression evaluates to $t$ with probability $1$. We have two cases. In the first case, $\hstar(x) = 1$, which implies that $\vstar(x) \le R$. Since $t = -1$, we write:
\begin{align*}
    \hhat(\patch{x}) &= \signv{\vstar(x) + t \cdot \frac{2R}{\min_{x \in \Sbackdoor} v_j(\patch{x})} \cdot v_j(\patch{x})} \\
    &= \signv{R- \frac{2R}{\min_{x \in \Sbackdoor} v_j(\patch{x})} \cdot v_j(\patch{x})} \\
    &= \signv{1 - \frac{2}{\min_{x \in \Sbackdoor} v_j(\patch{x})} \cdot v_j(\patch{x})} \\
    &= -1
\end{align*}
where the last line follows from the fact that  $1 - \frac{2}{\min_{x \in \Sbackdoor} v_j(\patch{x})} \cdot v_j(\patch{x}) \le -1$.

The proof in the event where $t = 1$ follows similarly. It thus follows that there exists a function $\hhat \in \cH$ obtaining $0$ error on $\Sclean \cup \Sbackdoor$. The desired result immediately follows from uniform convergence (we have a $0$-error classifier over $\Sclean \sim \cD$ and a $0$-error classifier over $\Sbackdoor \sim \patch{\cD | \hstar(x) \neq t}$, so with probability $1 - 2\delta$, we have error at most $\epsclean$ on the clean distribution and error at most $\epsadv$ on the adversarial distribution).
\end{proof}
}

Observe that in Theorem \ref{thm:existence_linear_backdoor}, if $\Sclean$ is sufficiently large, then $\Sbackdoor$ comprises a vanishingly small fraction of the training set. Therefore, the backdoor attack can succeed even when the fraction of corrupted examples in the training set is very small, so long as the \emph{quantity} of corrupted examples is sufficiently large.

\subsubsection{Overparameterized Linear Models}

To elucidate the scenarios subsumed by Theorem \ref{thm:existence_linear_backdoor}, consider the following example.

\newcommand{\corLinearBackdoor}{
Let $\cH$ be the set of linear separators over $\R^d$, and let $\cX = \R^d$. Let $\cD$ be some distribution over an $s$-dimensional subspace of $\R^d$ where $s < d$, so with probability $1$, we can write $x \sim \cD$ as $Az$ for some $A \in \R^{d \times s}$ and for $z \in \R^s$. Let $\fadv = \inbraces{\patch{x} \suchthat \patch{x} + \eta, \eta \perp \vspan{A}}$, and draw some patch function $\patchw \in \fadv$. 

Fix any target label $t \in \inbraces{\pm 1}$. Draw a training set $\Sclean$ of size at least $m_0 \coloneqq \Omega\inparen{\epsclean^{-1}\inparen{\vc{\cH} + \logv{\nfrac{1}{\delta}}}}$. Then, draw a backdoor training set $\Sbackdoor$ of size at least $m_1 \coloneqq \Omega\inparen{\epsadv^{-1}\inparen{\vc{\cH} + \logv{\nfrac{1}{\delta}}}}$ of the form $(x, t)$ where $x \sim \inparen{\cD | \hstar(x) \neq t} + \eta$.

With probability at least $1 - \delta$, empirical risk minimization on the training set $\Sclean \cup \Sbackdoor$ yields a classifier $\hhat$ satisfying the success conditions for Problem \ref{problem:adv_general}.
}
\begin{corollary}[Overparameterized Linear Classifier (Appendix Corollary \ref{app:linear_backdoor})]
\label{cor:linear_backdoor}
\corLinearBackdoor
\end{corollary}
\newcommand{\pfCorLinearBackdoor}{
\begin{proof}
We will show that our problem setup is a special case of that considered in Theorem \ref{thm:existence_linear_backdoor}; then, we can apply that result as a black box.

Observe that the set of linear classifiers over $\R^d$ is a thresholded vector space with dimension $d$. Pick the basis $\inbraces{v_1, \dots, v_s, \dots, v_d}$ such that $\inbraces{v_1, \dots, v_s}$ form a basis for the subspace $\vspan{A}$ and $v_{s + 1}, \dots, v_d$ are some completion of the basis for the rest of $\R^d$. 

Clearly, there is a size-$s$ set of indices $U \subset [d]$ such that for all $i \in U$, we have $\prvv{x \sim \cD}{v_i(x) \neq 0} > 0$. Without loss of generality, assume $U = [s]$.

Next, we need to show that for all $i \in U$, we have $v_i(\patch{x}) = 0$. Since we have $\eta \perp \vspan{A}$, we have $v_i(\eta) = 0$ for all $i \in U$. Since the $v_i$ are also linear functions, we satisfy $v_i(Az + \eta) = 0$ for all $z \in \R^s$. 

We now show that there is at least one $j \in \ucomplement$ such that $\prvv{x \sim \cD}{v_j(\patch{x}) \neq 0} = 1$. Since $\eta \perp \vspan{A}$, $\eta$ must be expressible as some nonzero linear combination of the vectors $v_j$; thus, taking the inner product with any such vector will result in a nonzero value.

Finally, we show that for all $j \in \ucomplement$, we either have $\prvv{x \sim \cD}{v_j(\patch{x}) \ge 0} = 1$ or $\prvv{x \sim \cD}{v_j(\patch{x}) \le 0} = 1$. Since $\eta$ is expressible as a linear combination of several such $v_j$, we can write:
\begin{align*}
    \ip{Az + \eta, v_j} &= \ip{Az, v_j} + \ip{\eta, v_j} \\
    &= 0 + \ip{\sum_{j = s + 1}^d a_j \cdot v_j, v_j} \\
    &= a_j
\end{align*}
which is clearly nonzero.

The result now follows from Theorem \ref{thm:existence_linear_backdoor}.
\end{proof}
}

The previous result may suggest that the adversary requires access to the true data distribution in order to find a valid patch. However, we can show that there exist conditions under which the adversary need not know even the support of the data distribution $\cD$. Informally, the next theorem states that if the degree of overparameterization is sufficiently high, then a \emph{random} stamp ``mostly'' lies in the orthogonal complement of $\vspan{A}$, and this is enough for a successful attack.

\newcommand{\thmRandomStamp}{
Consider the same setting used in Corollary \ref{cor:linear_backdoor}, and set $\fadv = \inbraces{\patchw \suchthat \patch{x} = x + \eta, \eta \in \R^d}$.

If $\hstar$ achieves margin $\gamma$ and if the ambient dimension $d$ of the model satisfies $d \ge \Omega\inparen{\nfrac{s\logv{s/\delta}}{\gamma^2}}$, then an adversary can find a patch function such that with probability $1-\delta$, a training set $S = \Sclean \cup \Sbackdoor$ satisfying $\abs{\Sclean} \ge \Omega\inparen{\epsclean^{-1}\inparen{\vc{\cH} + \logv{\nfrac{1}{\delta}}}}$ and $\abs{\Sbackdoor} \ge \Omega\inparen{\epsclean^{-1}\inparen{\vc{\cH} + \logv{\nfrac{1}{\delta}}}}$ yields a classifier $\hhat$ satisfying the success conditions for Problem \ref{problem:adv_general} while also satisfying $\exvv{(x,y)\sim\cD}{\indicator{\hhat(x)\neq y}} \le \epsclean$.

This result holds true particularly when the adversary does not know $\suppv{\cD}$.
}
\begin{theorem}[Random direction is an adversarial trigger (Appendix Theorem \ref{app:random_stamp})]
\label{thm:random_stamp}
\thmRandomStamp
\end{theorem}
\newcommand{\pfRandomStamp}{
\begin{proof}
We prove Theorem \ref{thm:random_stamp} in two parts. We first show that although the adversary doesn't know $\fadv(\hstar)$, they can find $\patchw \in \fadv(\hstar)$ with high probability. We then invoke the result from Corollary \ref{cor:linear_backdoor}.

Let $a_i$ denote the $i$th column of $A$. Next, draw $\eta$ from $\Unif{\S^{d-1}}$.

Recall that there exists a universal constant $C_0$ for which $\eta\sqrt{d}$ is $C_0$-subgaussian (\cite{Vershynin2018-xn}). Next, remember that if $\eta\sqrt{d}$ is $C_0$-subgaussian, then $\ip{\eta\sqrt{d}, a_i}$ has subgaussian constant $C_0\norm{a_i} = C_0$. Using classical subgaussian concentration inequalities, we arrive at the following:
\begin{align*}
    \prv{\abs{\ip{\eta\sqrt{d}, a_i}} \ge \frac{\eps\sqrt{d}}{\sqrt{s}}} &\le 2\expv{-\frac{\eps^2 d}{sC_0^2}} \\
    \implies \prv{\text{for all } i \in [s] \text{, } \abs{\ip{\eta, a_i}} \le \frac{\eps}{\sqrt{s}}} &\ge 1 - 2s\cdot\expv{-\frac{\eps^2 d}{sC_0^2}} \\
    &\ge 1 - \frac{\delta}{2} \quad \text{, pick } d = \frac{C_0^2}{\eps^2} \cdot s \cdot \logv{\frac{4s}{\delta}}
\end{align*}

Next, observe that if we have $\abs{\ip{\eta, a_i}} \le \nfrac{\eps}{\sqrt{s}}$ for all $i \in [s]$, then we have:
\begin{align*}
    \norm{A^T\eta} &= \sqrt{\sum_{i = 1}^{s} \abs{\ip{\eta, a_i}}^2} \\
    &\le \sqrt{\sum_{i = 1}^s \frac{\eps^2}{s}} \\
    &= \eps
\end{align*}
This implies that the norm of the component of the trigger in $\kernel{A^T}$ is at least $\sqrt{1-\eps^2} \ge 1 - \eps$ from the Pythagorean Theorem.

Next, we substitute $\eps = \gamma$. From this, we have that $\norm{A^Tv} \le \gamma$ with probability $1-\nfrac{\delta}{2}$, which implies that $\hstar(x + \eta) = \hstar(x)$ with probability $1-\nfrac{\delta}{2}$ over the draws of $\eta$. This gives us that $\patch{x} = x + \eta \in \fadv(\hstar)$ with probability $1-\nfrac{\delta}{2}$ over the draws of $\eta$.

It is now easy to see that the result we want follows from a simple application of Corollary \ref{cor:linear_backdoor} using a failure probability of $\nfrac{\delta}{2}$, and we're done, where the final failure probability $1-\delta$ follows from a union bound.
\end{proof}
}

Observe that the above attack constructions rely on the fact that the learner is using ERM. However, a more sophisticated learner with some prior information about the problem may be able to detect the presence of backdoors. Theorem \ref{thm:scale_capacity} gives an example of such a scenario.

\newcommand{\thmScaleCapacity}{
Consider some $\hstar(x) = \signv{\ip{\wstar,x}}$ and a data distribution $\cD$ satisfying $\prvv{(x,y)\sim\cD}{y\ip{\wstar, x} \ge 1} = 1$ and $\prvv{(x,y)\sim\cD}{\norm{x}\le R} = 1$. Let $\gamma$ be the maximum margin over all weight vectors classifying the uncorrupted data, and let $\fadv = \inbraces{\patch{x} \suchthat \norm{\patch{x} - x} \le \gamma}$. 

If $\Sclean$ consists of at least $\Omega\inparen{\epsclean^{-2}\inparen{\gamma^{-2}R^2 + \logv{\nfrac{1}{\delta}}}}$ i.i.d examples drawn from $\cD$ and if $\Sbackdoor$ consists of at least $\Omega\inparen{\epsadv^{-2}\inparen{\gamma^{-2}R^2 + \logv{\nfrac{1}{\delta}}}}$ i.i.d examples drawn from $\cD | \hstar(x) \neq t$, then we have:
$$\min_{w \suchthat \norm{w} \le \gamma^{-1}} \frac{1}{\abs{S}}\sum_{(x,y)\in S} \indicator{y\ip{w,x} < 1} > 0$$
In other words, assuming there exists a margin $\gamma$ and a $0$-loss classifier, empirical risk minimization of margin-loss with a norm constraint fails to find a $0$-loss classifier on a sufficiently contaminated training set.
}
\begin{theorem}[(Appendix Theorem \ref{app:scale_capacity})]
\label{thm:scale_capacity}
\thmScaleCapacity
\end{theorem}
\newcommand{\pfScaleCapacity}{
\begin{proof}
We will proceed by contradiction.

Let $\patch{x}$ denote the patched version of $x$. Without loss of generality, let the target label be $+1$. Set $\epsclean$ and $\epsadv$ such that $\epsclean + \epsadv < 1$ and draw enough samples such that the attack succeeds with parameters $\epsadv$ and $\delta$.

Observe that we can write every member in $\Sbackdoor$ as $(\patch{x}, y)$ for some natural $x$ with label $\neg y$. Next, suppose that the learner recovers a $\what$ such that the empirical margin loss of $\what$ is $0$. Next, recall that the following holds for $\what$ obtained from the minimization in the theorem statement and for a training set $S \sim \cD^m$ (see, for instance, Theorem 26.12 of \cite{Shalev-Shwartz2014-oj}):
$$\exvv{(x,y)\sim\cD}{\indicator{y\ip{\what, x} < 1}} \le \inf_{w \suchthat \norm{w} \le \gamma^{-1}}\exvv{(x,y)\sim S}{\indicator{y\ip{w,x} < 1}} + O\inparen{\sqrt{\frac{\inparen{\nfrac{R}{\gamma}}^2 + \logv{\nfrac{1}{\delta}}}{m}}}$$
Using this, it is easy to see that from uniform convergence, we have, with probability $1-\delta$:
\begin{align*}
    \prvv{x \sim \cD}{y\ip{\what, x} \ge 1} &\ge 1 - \epsclean \\
    \prvv{x \sim \cD}{\ip{\what, \patch{x}} \ge 1} &\ge 1 - \epsadv
\end{align*}
Thus, by a Union Bound, the following must be true:
$$\prvv{x \sim \cD}{\inparen{y\ip{\what, x} \ge 1} \wedge \inparen{\ip{\what, \patch{x}} \ge 1}} \ge 1 - \epsclean - \epsadv$$
Hence, it must be the case that there exists at least one true negative $x$ for which both $y\ip{\what, x} \ge 1$ and $\ip{\what, \patch{x}} \ge 1$ hold. We will use this to obtain a lower bound on $\norm{\what}$, from which a contradiction will follow. Notice that:
\begin{align*}
    1 &\le \ip{\what, \patch{x}} \\
    &= \ip{\what, x} + \ip{\what, \patch{x} - x} \\
    &\le -1 + \norm{\what} \cdot \norm{\patch{x} - x}
\end{align*}
where the last line follows from the fact that $x$ is labeled differently from $\patch{x}$. This gives:
$$\norm{\what} \ge \frac{2}{\norm{\patch{x} - x}}$$
Assuming that we meet the constraint $\norm{\what} \le \nfrac{1}{\gamma}$, putting the inequalities together gives:
$$\norm{\patch{x} - x} \ge 2\gamma$$
which is a contradiction, since we require that the size of the perturbation is smaller than the margin.
\end{proof}
}

\subsection{Memorization Capacity and Backdoor Attacks}

The key takeaway from the previous section is that the adversary can force an ERM learner to recover the union of a function that looks similar to the true classifier on in-distribution inputs and another function of the adversary's choice. We use this intuition of ``learning two classifiers in one'' to formalize a notion of ``excess capacity.''

To this end, we define the \emph{memorization capacity} of a class and a domain.

\begin{definition}[Memorization Capacity]
\label{defn:mcap}
Suppose we are in a setting where we are learning a hypothesis class $\cH$ over a domain $\cX$ under distribution $\cD$.

We say we can \emph{memorize} $k$ \emph{irrelevant} sets from a family $\cC$ atop a fixed $\hstar$ if we can find $k$ pairwise disjoint nonempty sets $X_1, \dots, X_k$ from a family of subsets of the domain $\cC$ such that for all $b \in \inbraces{\pm 1}^k$, there exists a classifier $\hhat \in \cH$ satisfying the below:
\begin{itemize}
    \item For all $x \in X_i$, we have $\hhat(x) = b_i$.
    \item $\prvv{x \sim \cD}{\hhat(x) = \hstar(x)} = 1$.
\end{itemize}
We define $\memcap{\cX,\cD}{h, \cH, \cC}$ to be the maximum number of sets from $\cC$ we can memorize for a fixed $h$ belonging to a hypothesis class $\cH$. We define $\memcap{\cX, \cD}{h, \cH} = \memcap{\cX,\cD}{h,\cH,\cB_\cX}$ to be the maximum number of sets from $\cB_\cX$ we can memorize for a fixed $h$, where $\cB_\cX$ is the family of all non-empty measurable subsets of $\cX$. Finally, we define $\memcap{\cX,\cD}{\cH} \coloneqq \sup_{h \in \cH} \memcap{\cX,\cD}{h, \cH}$.
\end{definition}

Intuitively, the memorization capacity captures the number of additional irrelevant (with respect to $\cD$) sets that can be memorized atop a true classifier. 

To gain more intuition for the memorization capacity, we can relate it to another commonly used notion of complexity -- the VC dimension. Specifically, we have the following lemma.

\newcommand{\lmMemVsVC}{We have $0 \le \memcap{\cX,\cD}{\cH} \le \vc{\cH}$.}
\begin{lemma}[(Appendix Lemma \ref{app:mem_vs_vc})]
\label{lemma:mem_vs_vc}
\lmMemVsVC
\end{lemma}
\newcommand{\pfMemVsVC}{
\begin{proof}
The lower bound is obvious. This is also tight, as we can set $\cX = \inbraces{0,1}^n$, $\cD = \mathsf{Unif}(\cX)$, and $\cH = \inbraces{f \suchthat f(x) = 1, \forall x \in \cX}$.

We now tackle the upper bound. Suppose for the sake of contradiction that $\memcap{\cX,\cD}{\cH} \ge \vc{\cH} + 1$. Then, we can find $k = \vc{\cH} + 1$ nonempty subsets of $\cX$, $X_1, \dots, X_k$ and an $h$ for which every labeling of these subsets can be achieved by some other $\hhat \in \cH$. Hence, picking any collection of points $x_i \in X_i$ yields a set witnessing $\vc{\cH} \ge k = \vc{\cH} + 1$, which is clearly a contradiction.

The upper bound is tight as well. Consider the dataset $S = \inbraces{0, e_1, \dots, e_d}$, let $\cD$ be a distribution assigning a point mass of $1$ to $x = 0$, and let $\hstar(0) = 1$. It is easy to see that the class of origin-containing halfspaces can memorize every labeling $e_1, \dots, e_d$ as follows -- suppose we have labels $b_1, \dots, b_d$. Then, the classifier:
$$\indicator{\sum_{i = 1}^d b_i \cdot x_i \ge 0}$$
memorizes every labeling of $e_1, \dots, e_d$ while correctly classifying the pair $(0, 1)$. Hence, we can memorize $d$ irrelevant sets, which is equal to the VC dimension of origin-containing linear separators.
\end{proof}
}

Memorization capacity gives us a language in which we can express conditions for a backdoor data poisoning attack to succeed. Specifically, we have the following general result.

\newcommand{\thmMcapAttack}{
Pick a target label $t \in \pm 1$. Suppose we have a hypothesis class $\cH$, a target function $\hstar$, a domain $\cX$, a data distribution $\cD$, and a class of patch functions $\fadv$. Define:
\begin{align*}
    \cC(\fadv(\hstar)) \coloneqq \{\patch{\suppv{\cD | \hstar(x) \neq t}} \suchthat \patchw \in \fadv\}
\end{align*}
Now, suppose that $\memcap{\cX,\cD}{\hstar,\cH,\cC(\fadv(\hstar))} \ge 1$. Then, there exists a function $\patchw \in \fadv$ for which the adversary can draw a set $\Sbackdoor$ consisting of $m = \Omega\inparen{\epsadv^{-1}\inparen{\vc{\cH} + \logv{\nfrac{1}{\delta}}}}$ i.i.d samples from $\cD | \hstar(x) \neq t$ such that with probability at least $1-\delta$ over the draws of $\Sbackdoor$, the adversary achieves the objectives of Problem \ref{problem:adv_general}, regardless of the number of samples the learner draws from $\cD$ for $\Sclean$.
}
\newcommand{\thmMcapAttackGeneral}{
Pick an array of $k$ target labels $t \in \inbraces{\pm 1}^k$. Suppose we have a hypothesis class $\cH$, a target function $\hstar$, a domain $\cX$, a data distribution $\cD$, and a class of patch functions $\fadv$. Define:
\begin{align*}
    \cC(\fadv(\hstar))_{t'} \coloneqq \{\patch{\suppv{\cD | \hstar(x) \neq t'}} \suchthat &\patchw \in \fadv\}
\end{align*}
and let:
$$\cC(\fadv(\hstar)) \coloneqq \cC(\fadv(\hstar))_{-1} \cup \cC(\fadv(\hstar))_{1}$$
Now, suppose that $\memcap{\cX,\cD}{\hstar,\cH,\cC(\fadv(\hstar))} \ge k$. Then, there exists $k$ functions $\patchw_1, \dots, \patchw_k \in \fadv$ for which the adversary can draw sets $\inbraces{(\Sbackdoor)_i}_{i \in [k]}$ each consisting of $m_i = \Omega\inparen{\epsadv^{-1}\inparen{\vc{\cH} + \logv{\nfrac{k}{\delta}}}}$ i.i.d samples from $\cD | \hstar(x) \neq t_i$ such that with probability at least $1-\delta$ over the draws of $(\Sbackdoor)_i$, the adversary achieves the objectives of Problem \ref{problem:adv_general}, regardless of the number of samples the learner draws from $\cD$ for $\Sclean$.
}
\begin{theorem}[Nonzero Memorization Capacity Implies Backdoor Attack (Appendix Theorem \ref{app:mcap_attack})]
\label{thm:mcap_attack}
\thmMcapAttack
\end{theorem}
\newcommand{\pfMcapAttack}{
\begin{proof}
As per the theorem statement, we can draw $m$ samples from $\cD | \hstar(x) \neq t_i$ to form $\Sbackdoor$ by inverting the labels of the samples we draw. 

Since $\memcap{\cX,\cD}{\hstar, \cH, \cC(\fadv(\hstar))} = k$, there must exist $k$ sets $X_1, \dots, X_k \in \cC(\fadv(\hstar))$ such that the $X_i$ are memorizable, for which we can write $X_i \subseteq \patchw_i\inparen{\suppv{\cD | \hstar(x) \neq t_i}}$ for appropriate choices of $\patchw_i$, and for which $\mu_{\patch{\cD | \hstar(x) \neq t_i}}(X_i) = 1$. This implies that with probability $1$, there exists at least one function $\hhat \in \cH$ such that $\hhat$ returns $t_i$ on every element in $(\Sbackdoor)_i$ for all $i \in [k]$ and agrees with $\hstar$ on every element in the clean training set $\Sclean$.

Thus, we can recover a classifier $\hhat$ from $\cH$ with $0$ error on the training set $\Sclean \cup \inparen{\bigcup_{i \in [k]} (\Sbackdoor)_i}$. In particular, notice that we achieve $0$ error on $\Sclean$ from distribution $\cD$ and on every $(\Sbackdoor)_i$ from distribution $\patchw_i\inparen{\cD | \hstar(x) \neq t_i}$. From the Fundamental Theorem of PAC Learning (\cite{Shalev-Shwartz2014-oj}), it follows that as long as $\abs{\Sclean}$ and $\abs{(\Sbackdoor)_i}$ are each at least $\Omega\inparen{\epsclean^{-1}\inparen{\vc{\cH} + \logv{\nfrac{k}{\delta}}}}$ and $\Omega\inparen{\epsadv^{-1}\inparen{\vc{\cH} + \logv{\nfrac{k}{\delta}}}}$, respectively, we have that $\hhat$ has error at most $\eps$ on $\cD$ and error at least $1 - \eps$ on $\patchw_i\inparen{\cD | \hstar(x) \neq t_i}$ with probability $1-\delta$ (following from a union bound, where each training subset yields a failure to attain uniform convergence with probability at most $\nfrac{\delta}{(k + 1)}$).
\end{proof}
}

In words, the result of Theorem \ref{thm:mcap_attack} states that nonzero memorization capacity with respect to subsets of the images of valid patch functions implies that a backdoor attack exists. More generally, we can show that a memorization capacity of at least $k$ implies that the adversary can \emph{simultaneously} execute $k$ attacks using $k$ different patch functions. In practice, this could amount to, for instance, selecting $k$ different triggers for an image and correlating them with various desired outputs. We defer the formal statement of this more general result to the Appendix (see Appendix Theorem \ref{app:mcap_attack_general}).

A natural follow-up question to the result of Theorem \ref{thm:mcap_attack} is to ask whether a memorization capacity of zero implies that an adversary cannot meet its goals as stated in Problem \ref{problem:adv_general}. Theorem \ref{thm:mcap_zero} answers this affirmatively.

\newcommand{\thmMcapZero}{
Let $\cC(\fadv(\hstar))$ be defined the same as in Theorem \ref{thm:mcap_attack}. Suppose we have a hypothesis class $\cH$ over a domain $\cX$, a true classifier $\hstar$, data distribution $\cD$, and a perturbation class $\fadv$. If $\memcap{\cX,\cD}{\hstar, \cH, \cC(\fadv(\hstar))} = 0$, then the adversary cannot successfully construct a backdoor data poisoning attack as per the conditions of Problem \ref{problem:adv_general}.
}
\begin{theorem}[Nonzero Memorization Capacity is Necessary for Backdoor Attack (Appendix Theorem \ref{app:mcap_zero})]
\label{thm:mcap_zero}
\thmMcapZero
\end{theorem}
\newcommand{\pfMcapZero}{
\begin{proof}
The condition in the theorem statement implies that there does not exist an irrelevant set that can be memorized atop any choice of $h \in \cH$.

For the sake of contradiction, suppose that there does exist a target classifier $\hstar$, a function $\patchw \in \fadv$ and a target label $t$ such that for all choices of $\epsclean$, $\epsadv$, and $\delta$, we obtain a successful attack.

Define the set $X \coloneqq \patch{\suppv{\cD | \hstar(x) \neq t}}$; in words, $X$ is the subset of $\cX$ consisting of patched examples that are originally of the opposite class of the the target label. It is easy to see that $X \in \cC$.

We will first show that if $\mu_\cD(X) > 0$, then we obtain a contradiction. Set $0 < \epsadv, \epsclean < \frac{\mu_\cD(X)}{1 + \mu_\cD(X)}$. Since the attack is successful, we must classify at least a $1 - \epsadv$ fraction of $X$ as the target label. Hence, we can write:
\begin{align*}
    \mu_\cD\inparen{\inbraces{x \in X \suchthat \hhat(x) = t}} &\ge \inparen{1 - \epsadv}\mu_\cD(X) \\
    &> \frac{1}{1 + \mu_\cD(X)} \cdot \mu_\cD(X) \\
    &> \epsclean
\end{align*}
Since the set $\inbraces{x \in X \suchthat \hhat(x) = t}$ is a subset of the region of $\cX$ that $\hhat$ makes a mistake on, we have that $\hhat$ must make a mistake on at least $\epsclean$ measure of $\cD$, which is a contradiction.

Hence, it must be the case that $\mu_\cD(X) = 0$; in other words, $X$ is an irrelevant set. Recall that in the beginning of the proof, we assume there exists a function $\hhat$ that achieves label $t$ on $X$, which is opposite of the value of $\hstar$ on $X$. Since we can achieve both possible labelings of $X$ with functions from $\cH$, it follows that $X$ is a memorizable set, and thus the set $X$ witnesses positive $\memcap{\cX,\cD}{\hstar, \cH, \cC(\fadv(\hstar))}$.
\end{proof}
}

\subsubsection{Examples}

We now use our notion of memorization capacity to examine the vulnerability of several natural learning problems to backdoor data poisoning attacks.

\newcommand{\exLinearBackdoorMcap}{Recall the result from the previous section, where we took $\cX = \R^d$, $\cH_d$ to be the set of linear classifiers in $\R^d$, and let $\cD$ be a distribution over a radius-$R$ subset of an $s$-dimensional subspace $P$. We also assume that the true labeler $\hstar$ achieves margin $\gamma$.
 
If we set $\fadv = \inbraces{\patch{x} \suchthat \patch{x} = x + \eta, \eta \in \R^d}$, then we have $\memcap{\cX,\cD}{\hstar, \cH_d, \cC(\fadv(\hstar))} \ge d - s$. }
\begin{example}[Overparameterized Linear Classifiers (Appendix Example \ref{app:linear_backdoor_mcap})]
\label{ex:linear_backdoor_mcap}
\exLinearBackdoorMcap
\end{example}
\newcommand{\pfLinearBackdoorMcap}{
\begin{proof}
Let $\wstar$ be the weight vector corresponding to $\hstar$.

Observe that there exists $k \coloneqq d - s$ unit vectors $v_1, \dots, v_k$ that complete an orthonormal basis from that for $P$ to one for $\R^d$. Next, consider the following subset of $\fadv(\hstar)$:
$$\fadv' \coloneqq \inbraces{\patchw \in \fadv \suchthat \forall i \in [k], \patchw_i\inparen{x} = \inparen{\begin{cases} x + \eta \cdot t_iv_i &, \hstar(x) \neq t_i \\ x & \text{otherwise} \end{cases}}}$$

We prove the memorization capacity result by using the images of functions in $\fadv'$. We will show that the function:
$$\hhat(x) = \signv{\ip{\wstar + \frac{2R}{\gamma}\sum_{i=1}^k t_i\cdot\frac{v_i}{\eta_i}, x}}$$
memorizes the $k$ sets $C_i \coloneqq \inbraces{x + \eta_i \cdot v_i \suchthat \ip{\wstar, x} \in \insquare{1, \nfrac{R}{\gamma}} \cup \insquare{-\nfrac{R}{\gamma}, -1}}$. Moreover, observe that the preimages of the $C_i$ have measure $1$ under the conditional distributions $\cD | \hstar(x) \neq t_i$, since the preimages contain the support of these conditional distributions. We now have that, for a clean point $x \in P$:
\begin{align*}
    \hhat(x) &= \signv{\ip{\wstar + \frac{2R}{\gamma}\sum_{i=1}^k t_i\cdot\frac{v_i}{\eta_i}, x}} \\
    &= \signv{\ip{\wstar, x} + \frac{2R}{\gamma}\ip{\sum_{i=1}^k t_i\cdot\frac{v_i}{\eta_i}, x}} \\
    &= \signv{\ip{\wstar, x}} = \hstar(x)
\end{align*}
and for a corrupted point $x + \eta_j \cdot v_j$, for $j \in [k]$:
\begin{align*}
    \hhat(x) &= \signv{\ip{\wstar + \frac{2R}{\gamma}\sum_{i=1}^k t_i\cdot\frac{v_i}{\eta_i}, x + \eta_j \cdot v_j}} \\
    &= \signv{\ip{\wstar, x + \eta_j \cdot v_j} + \frac{2R}{\gamma}\ip{\sum_{i=1}^k t_i\cdot\frac{v_j}{\eta_j}, x + \eta_j \cdot v_j}} \\
    &= \signv{\ip{\wstar, x} + \frac{2R}{\gamma}\ip{\sum_{i=1}^k t_i\cdot\frac{v_i}{\eta_i}, x} + \frac{2R}{\gamma}\ip{\sum_{i=1}^k t_i\cdot\frac{v_i}{\eta_i}, \eta_j \cdot v_j}} \\
    &= \signv{\insquare{\pm \frac{R}{\gamma}} + t_j \cdot \frac{2R}{\gamma}} \\
    &= t_j
\end{align*}
This shows that we can memorize the $k$ sets $C_i$. It is easy to see that $\mu_\cD(C_i) = 0$, so the $C_i$ are irrelevant memorizable sets; in turn, we have that $\memcap{\cX,\cD}{\hstar} \ge k = d - s$, as desired.
\end{proof}
}

\newcommand{\exLinearBackdoorCvx}{Let $\cH$ be the set of origin-containing halfspaces. Fix an origin-containing halfspace $\hstar$ with weight vector $\wstar$. Let $\cX'$ be a closed compact convex set, let $\cX = \cX'  \setminus \inbraces{x \suchthat \ip{\wstar, x} = 0}$, and let $\cD$ be any probability measure over $\cX$ that assigns nonzero measure to every $\ell_2$ ball of nonzero radius contained in $\cX$ and satisfies the relation $\mu_\cD(Y) = 0 \iff \mathsf{Vol}_d(Y) = 0$ for all $Y \subset \cX$. Then, $\memcap{\cX,\cD}{\hstar, \cH} = 0$.}
\begin{example}[Linear Classifiers Over Convex Bodies (Appendix Example \ref{app:linear_backdoor_cvx})]
\label{ex:linear_backdoor_cvx}
\exLinearBackdoorCvx
\end{example}
\newcommand{\pfLinearBackdoorCvx}{
\begin{proof}
Observe that it must be the case that the dimension of the ambient space is equal to the dimension of $\cX$.

Let $\wstar$ be the weight vector corresponding to the true labeler $\hstar$.

For the sake of contradiction, suppose there exists a classifier $\what$ satisfying $\prvv{x\sim\cD}{\signv{\ip{\what, x}} = \signv{\ip{\wstar, x}}} = 1$, but there exists a subset $Y \subset \cX$ for which $\signv{\ip{\what, x}} \neq \signv{\ip{\wstar, x}}$, for all $x \in Y$. Such a $Y$ would constitute a memorizable set.

Without loss of generality, let the target label be $-1$; that is, the adversary is converting a set $Y$ whose label is originally $+1$ to one whose label is $-1$. Additionally, without loss of generality, take $\norm{\wstar} = \norm{\what} = 1$. Observe that the following set relationship must hold:
$$Y \subseteq D \coloneqq \inbraces{x \in \cX \suchthat \ip{\what, x} \le 0 \text{ and } \ip{\wstar, x} > 0}$$
For $D$ to be nonempty (and therefore for $Y$ to be nonempty), observe that we require $\what \neq \wstar$ (otherwise, the constraints in the definition of the set $D$ are unsatisfiable). 

We now need the following intermediate result.
\begin{lemma*}
Consider some convex body $K$, a probability measure $\cD$ such that every $\ell_2$ ball of nonzero radius within $K$ has nonzero measure, and some subset $K' \subseteq K$ satisfying $\mu_\cD(K') = 1$. Then, $\conv{K'}$ contains every interior point of $K$.
\end{lemma*}
\begin{proof}
Recall that an interior point is defined as one for which we can find some neighborhood contained entirely within the convex body. Mathematically, $x \in K$ is an interior point if we can find nonzero $\delta$ for which $\inbraces{z \suchthat \norm{x - z} \le \delta} \subseteq K$ (see \cite{Boyd2004-ey}).

For the sake of contradiction, suppose that there exists some interior point $x \in K$ that is not contained in $\conv{K'}$. Hence, there must exist a halfspace $H$ with boundary passing through $x$ and entirely containing $\conv{K'}$. Furthermore, there must exist a nonzero $\delta$ for which there is an $\ell_2$ ball centered at $x$ of radius $\delta$ contained entirely within $K$. Call this ball $B_2(x, \delta)$. Thus, the set $K \setminus H$ cannot be in $\conv{K'}$. 

We will now show that $\mu_\cD(K \setminus H) > 0$. Observe that the hyperplane inducing $H$ must cut $B_2(x, \delta)$ through an equator. From this, we have that the set $K \setminus H$ contains a half-$\ell_2$ ball of radius $\delta$. It is easy to see that this half-ball contains another $\ell_2$ ball of radius $\nfrac{\delta}{2}$ (call this $B'$), and as per our initial assumption, $B'$ must have nonzero measure.

Thus, we can write $\mu_\cD(K \setminus H) \ge \mu_\cD(B') > 0$. Since we know that $\mu_\cD(\conv{K'}) + \mu_\cD(K \setminus H) \le 1$, it follows that $\mu_\cD(\conv{K'}) < 1$ and therefore $\mu_\cD(K') < 1$, violating our initial assumption that $\mu_\cD(K') = 1$.
\end{proof}

This lemma implies that if $Y$ is memorizable, then it must lie entirely on the boundary of the set $\cX_{+} \coloneqq \inbraces{x \in \cX \suchthat \ip{\wstar, x} > 0}$. To see this, observe that if $\what$ classifies any (conditional) measure-$1$ subset of $\cX_{+}$ correctly, then it must classify the convex hull of that subset correctly as well. This implies that $\what$ must correctly classify every interior point in $\cX_{+}$, and thus, $Y$ must be entirely on the boundary of $\cX_{+}$.

We will now show the following intermediate result.
\begin{lemma*}
Let $K$ be a closed compact convex set. Let $x_1$ be on the boundary of $K$ and let $x_2$ be an interior point of $K$. Then, every point of the form $\lambda x_1 + (1 - \lambda) x_2$ for $\lambda \in (0, 1)$ is an interior point of $K$.
\end{lemma*}
\begin{proof}
Since $x_2$ is an interior point, there must exist an $\ell_2$ ball of radius $\delta$ contained entirely within $K$ centered at $x_2$. From similar triangles and the fact that any two points in a convex body can be connected by a line contained in the convex body, it is easy to see that we can center an $\ell_2$ ball of radius $(1 - \lambda)\delta$ at the point $\lambda x_1 + (1 - \lambda) x_2$ that lies entirely in $K$. This is what we wanted, and we're done.
\end{proof}

Now, let $x_1 \in Y$ and $x_2 \in \mathsf{Interior}(\cX_{-})$ where $\cX_{-} = \inbraces{x \in \cX \suchthat \ip{\wstar, x} < 0}$. Draw a line from $x_1$ to $x_2$ and consider the labels of the points assigned by $\what$. Since $x_1 \in Y$, we have $\hhat(x_1) = -1$, and since $x_2 \in \mathsf{Interior}(\cX_{-})$, we have that $\hhat(x_2) = -1$ as well. Using our lemma, we have that every point on the line connecting $x_1$ to $x_2$ (except for possibly $x_1$) is an interior point to $\cX'$. Since we have that the number of sign changes along a line that can be induced by a linear classifier is at most $1$, we must have that the line connecting $x_1$ to $x_2$ incurs $0$ sign changes with respect to the classifier induced by $\what$. This implies that the line connecting $x_1$ to $x_2$ cannot pass through any interior points of $\cX_{+}$. However, the only way that this can happen is if $\ip{\wstar, x_1} = 0$, but per our definition of $\cX$, if it is the case that $\ip{\wstar, x_1} = 0$, then $x_1 \notin \cX$, which is a clear contradiction.

This is sufficient to conclude the proof, and we're done.
\end{proof}
}

Given these examples, it is natural to wonder whether memorization capacity can be greater than $0$ when the support of $\cD$ is the entire space $\cX$. The following example shows this indeed can be the case.

\newcommand{\exSignChanges}{Let $\cX = \insquare{0,1}$, $\cD  = \Unif{\cX}$ and $\cH_k$ be the class of functions admitting at most $k$ sign-changes. Specifically, $\cH_k$ consists of functions $h$ for which we can find pairwise disjoint, continuous intervals $I_1, \dots, I_{k+1}$ such that:
\begin{itemize}
    \item For all $i < j$ and for all $x \in I_i, y \in I_j$, we have $x < y$.
    \item $\bigcup_{i=1}^{k+1} I_i = \cX$.
    \item $h(I_i) = -h(I_{i+1})$, for all $i \in [k]$.
\end{itemize}
Suppose the learner is learning $\cH_s$ for unknown $s$ using $\cH_d$, where $s \le d+2$. For all $\hstar \in \cH_s$, we have $\memcap{\cX,\cD}{\hstar, \cH_d} \ge \floor{\nfrac{(d-s)}{2}}$.}
\begin{example}[Sign Changes (Appendix Example \ref{app:sign_changes})]
\label{ex:sign_changes}
\exSignChanges
\end{example}
\newcommand{\pfSignChanges}{\begin{proof}
Without loss of generality, take $d-s$ to be an even integer.

Let $I_1, \dots, I_{s+1}$ be the intervals associated with $\hstar$. It is easy to see that we can pick a total of $\nfrac{(d-s)}{2}$ points such that the sign of these points can be memorized by some $\hhat$. Since each point we pick within an interval can induce at most $2$ additional sign changes, we have that the resulting function $\hhat$ has at most $s + 2 \cdot \nfrac{(d-s)}{2} \le d$ sign-changes; thus, $\hhat \in \cH_d$. Moreover, the measure of a single point is $0$, and so the total measure of our $\nfrac{(d-s)}{2}$ points is $0$.

Given this, it is easy to find $\fadv$ and corresponding $\cC(\fadv(\hstar))$ for which the backdoor attack can succeed as per Theorem \ref{thm:mcap_attack}.
\end{proof}}

\section{Algorithmic Considerations}
\label{sec:algorithmic}

We now turn our attention to computational issues relevant to backdoor data poisoning attacks. Throughout the rest of this section, define the adversarial loss:
$$\cL_{\fadv(\hstar)}(\hhat, S) \coloneqq \exvv{(x,y) \sim S}{\sup_{\patchw \in \fadv(\hstar)} \indicator{\hhat(\patch{x}) \neq y}}$$
In a slight overload of notation, let $\robustloss^\cH$ denote the robust loss class of $\cH$ with the perturbation sets generated by $\fadv(\hstar)$:
$$\robustloss^\cH \coloneqq \inbraces{(x,y) \mapsto \sup_{\patchw \in \fadv(\hstar)} \indicator{\hhat(\patch{x}) \neq y} \suchthat \hhat \in \cH}$$
Then, assume that $\vc{\robustloss^\cH}$ is finite\footnote{It is shown in \cite{Montasser2019-ro} that there exist classes $\cH$ and corresponding adversarial loss classes $\robustloss$ for which $\vc{\cH} < \infty$ but $\vc{\robustloss^\cH} = \infty$. Nonetheless, there are a variety of natural scenarios in which we have $\vc{\cH}, \vc{\robustloss^\cH} < \infty$; for example, in the case of linear classifiers in $\R^d$ and for closed, convex, origin-symmetric, additive perturbation sets, we have $\vc{\cH}, \vc{\robustloss^\cH} \le d + 1$ (see \cite{Montasser2020-ir} \cite{Cullina2018-os}).}. Finally, assume that the perturbation set $\fadv$ is the same as that consistent with the ground-truth classifier $\hstar$. In other words, once $\hstar$ is selected, then we reveal to both the learner and the adversary the sets $\fadv(\hstar)$; thus, the learner equates $\fadv$ and $\fadv(\hstar)$. Hence, although $\hstar$ is not known to the learner, $\fadv(\hstar)$ is. As an example of a natural scenario in which such an assumption holds, consider the case where $\hstar$ is some large-margin classifier and $\fadv$ consists of short additive perturbations. This subsumes the setting where $\hstar$ is some image classifier and $\fadv$ consists of test-time adversarial perturbations which don't impact the true classifications of the source images.

\subsection{Certifying the Existence of Backdoors}

The assumption that $\fadv = \fadv(\hstar)$ gives the learner enough information to minimize $\cL_{\fadv(\hstar)}(\hhat, S)$ on a finite training set $S$ over $\hhat \in \cH$\footnote{However, minimizing $\robustloss$ might be computationally intractable in several scenarios.}; the assumption that $\vc{\robustloss^\cH} < \infty$ yields that the learner recovers a classifier that has low robust loss as per uniform convergence. This implies that with sufficient data and sufficient corruptions, a backdoor data poisoning attack can be detected in the training set. We formalize this below.

\newcommand{\thmCertifyBackdoor}{Suppose that the learner can calculate and minimize:
$$\cL_{\fadv(\hstar)}(\hhat, S) = \exvv{(x,y) \sim S}{\sup_{\patchw \in \fadv(\hstar)} \indicator{\hhat(\patch{x}) \neq y}}$$
over a finite set $S$ and $\hhat \in \cH$.

If the VC dimension of the loss class $\robustloss^\cH$ is finite, then there exists an algorithm using  $O\inparen{\epsclean^{-2}\inparen{\vc{\robustloss} + \logv{\nfrac{1}{\delta}}}}$ samples that allows the learner to defeat the adversary through learning a backdoor-robust classifier or by rejecting the training set as being corrupted, with probability $1-\delta$.}
\begin{theorem}[Certifying Backdoor Existence (Appendix Theorem \ref{app:certify_backdoor})]
\label{thm:certify_backdoor}
\thmCertifyBackdoor
\end{theorem}
\newcommand{\pfCertifyBackdoor}{\begin{proof}

See Algorithm \ref{alg:certify_backdoor} for the pseudocode of an algorithm witnessing Theorem \ref{thm:filter_to_generalize}. 

\begin{algorithm}[ht!]
\caption{Implementation of an algorithm certifying backdoor corruption\label{alg:certify_backdoor}}
\begin{algorithmic}[1]
    \STATE \textbf{Input}: Training set $S = \Sclean \cup \Sbackdoor$\\ satisfying $\abs{\Sclean} = \Omega\inparen{\epsclean^{-2}\inparen{\vc{\robustloss^\cH} + \logv{\nfrac{1}{\delta}}}}$
    \STATE Set $\hhat \coloneqq \argmin_{h \in \cH} \robustloss(h, S)$
    \STATE \textbf{Output}: $\hhat$ if $\robustloss(\hhat, S) \le 2\eps$ and reject otherwise
\end{algorithmic}
\end{algorithm}

There are two scenarios to consider.

\paragraph{Training set is (mostly) clean.} Suppose that $S$ satisfies $\min_{h \in \cH} \robustloss(h, S) \lesssim \epsclean$. Since the VC dimension of the loss class $\robustloss^\cH$ is finite, it follows that with finitely many samples, we attain uniform convergence with respect to the robust loss, and we're done; in particular, we can write $\robustloss\inparen{\argmin_{h\in\cH}\robustloss(h, S), \cD} \lesssim \epsclean$ with high probability.

\paragraph{Training set contains many backdoored examples.} Here, we will show that with high probability, minimizing $\cL_{\fadv(\hstar)}(\hhat, S)$ over $\hhat$ will result in a nonzero loss, which certifies that the training set $S$ consists of malicious examples.

Suppose that for the sake of contradiction, the learner finds a classifier $\hhat$ such that $\cL_{\fadv(\hstar)}(\hhat, S) \lesssim \epsclean$. Hence, with high probability, we satisfy $\cL_{\fadv(\hstar)}(\hhat, \cD) \lesssim \epsclean$. Since there is a constant measure allocated to each class, we can write:
$$\exvv{(x,y)\sim\cD | y\neq t}{\sup_{\patchw\in\fadv(\hstar)} \indicator{\hhat(\patch{x}) \neq y}} \lesssim \epsclean$$
Furthermore, since we achieved a loss of $0$ on the whole training set, including the subset $\Sbackdoor$, from uniform convergence, we satisfy the following with high probability:
$$\exvv{(x,y)\sim\cD | y\neq t}{\indicator{\hhat(\patch{x}) = t}} \ge 1-\epsadv$$
which immediately implies:
$$\exvv{(x,y)\sim\cD | y\neq t}{\sup_{\patchw\in\fadv(\hstar)} \indicator{\hhat(\patch{x}) \neq y}} \ge 1 - \epsadv$$
Chaining the inequalities together yields:
$$\epsclean \gtrsim 1 - \epsadv$$
which is a contradiction, as we can make $\epsclean$ sufficiently small so as to violate this statement.
\end{proof}}

See Algorithm \ref{alg:certify_backdoor} in the Appendix for the pseudocode of an algorithm witnessing the statement of Theorem \ref{thm:certify_backdoor}.

Our result fleshes out and validates the approach implied by \cite{Borgnia2021-vk}, where the authors use data augmentation to robustly learn in the presence of backdoors. Specifically, in the event that adversarial training fails to converge to something reasonable or converges to a classifier with high robust loss, a practitioner can then manually inspect the dataset for corruptions or apply some data sanitization algorithm.


\subsubsection{Numerical Trials}

To exemplify such a workflow, we implement adversarial training in a backdoor data poisoning setting. Specifically, we select a target label, inject a varying fraction of poisoned examples into the MNIST dataset (see \cite{lecun-mnisthandwrittendigit-2010}), and estimate the robust training and test loss for each choice of $\alpha$. Our results demonstrate that in this setting, the training robust loss indeed increases with the fraction of corrupted data $\alpha$; moreover, the classifiers obtained with low training robust loss enjoy a low test-time robust loss. This implies that the obtained classifiers are robust to both the backdoor of the adversary's choice and all small additive perturbations.

For a more detailed description of our methodology, setup, and results, please see Appendix Section \ref{sec:numerical_trials}.

\subsection{Filtering versus Generalization}

We now show that two related problems we call \emph{backdoor filtering} and \emph{robust generalization} are nearly statistically equivalent; computational equivalence follows if there exists an efficient algorithm to minimize $\robustloss$ on a finite training set. We first define these two problems below (Problems \ref{problem:filtering} and \ref{problem:generalization}).

\begin{problem}[Backdoor Filtering]
\label{problem:filtering}
Given a training set $S = \Sclean \cup \Sbackdoor$ such that $\abs{\Sclean} \ge \Omega\inparen{\mathsf{poly}\inparen{\eps^{-1}, \logv{\nfrac{1}{\delta}}, \vc{\robustloss}}}$, return a subset $S' \subseteq S$ such that the solution to the optimization $\hhat \coloneqq \argmin_{h \in \cH} \robustloss\inparen{h, S'}$
satisfies $\robustloss(h, \cD) \lesssim \epsclean$ with probability $1-\delta$.
\end{problem}

Informally, in the filtering problem (Problem \ref{problem:filtering}), we want to filter out enough backdoored examples such that the training set is clean enough to obtain robust generalization.

\begin{problem}[Robust Generalization]
\label{problem:generalization}
Given a training set $S = \Sclean \cup \Sbackdoor$ such that $\abs{\Sclean} \ge \Omega\inparen{\mathsf{poly}\inparen{\eps^{-1}, \logv{\nfrac{1}{\delta}}, \vc{\robustloss}}}$, return a classifier $\hhat$ satisfies $\robustloss{\hhat, \cD} \le \epsclean$ with probability $1-\delta$.
\end{problem}

In other words, in Problem \ref{problem:generalization}, we want to learn a classifier robust to all possible backdoors.

In the following results (Theorems \ref{thm:filter_to_generalize} and \ref{thm:generalize_to_filter}), we show that Problems \ref{problem:filtering} and \ref{problem:generalization} are statistically equivalent, in that a solution for one implies a solution for the other. Specifically, we can write the below.

\newcommand{\pfFractionSubtraction}{
Let $a, b, c$ be positive numbers. We first write:
$$\frac{a}{b} - \max\inbraces{0,\frac{a-c}{b-c}} = \frac{c(b-a)}{b(b-c)} \le \frac{c}{b}$$
which occurs exactly when $c \le a$. In case where $a \le c$:
$$\frac{a}{b} - \max\inbraces{0,\frac{a-c}{b-c}} = \frac{a}{b} \le \frac{c}{b}$$
which gives:
$$\frac{a}{b} - \max\inbraces{0,\frac{a-c}{b-c}} \le \frac{c}{b}$$
Next, consider:
$$\min\inbraces{1, \frac{a}{b-c}} - \frac{a}{b}  = \frac{a}{b-c} - \frac{a}{b} = \frac{c}{b} \cdot \frac{a}{b - c} \le \frac{c}{b}$$
which happens exactly when we have $b \ge a + c$. In the other case:
$$\min\inbraces{1, \frac{a}{b-c}} - \frac{a}{b}  = 1 - \frac{a}{b} \le \frac{c}{b}$$
We can thus write:
$$\max\inbraces{0, \frac{a-c}{b-c}}, \min\inbraces{1, \frac{a}{b-c}} \in \insquare{\frac{a}{b} \pm \frac{c}{b}}$$
}

\newcommand{\thmFilterToGeneralize}{Let $\alpha \le \nfrac{1}{3}$ and $\epsclean \le \nfrac{1}{10}$.

Suppose we have a training set $S = \Sclean \cup \Sbackdoor$ such that $\abs{\Sclean} = \Omega\inparen{\epsclean^{-2}\inparen{\vc{\robustloss} + \logv{\nfrac{1}{\delta}}}}$ and $\abs{\Sbackdoor} \le \alpha \cdot \inparen{\abs{\Sbackdoor} + \abs{\Sclean}}$. If there exists an algorithm that given $S$ can find a subset $S' = \Sclean' \cup \Sbackdoor'$ satisfying $\nfrac{\abs{\Sclean'}}{\abs{\Sclean}} \ge 1-\epsclean$ and $\min_{h\in\cH}\robustloss(h, S') \lesssim \epsclean$, then there exists an algorithm such that given $S$ returns a function $\hhat$ satisfying $\robustloss(\hhat, \cD) \lesssim \epsclean$ with probability $1-\delta$.}
\begin{theorem}[Filtering Implies Generalization (Appendix Theorem \ref{app:filter_to_generalize})]
\label{thm:filter_to_generalize}
\thmFilterToGeneralize
\end{theorem}
\newcommand{\pfFilterToGeneralize}{\begin{proof}
See Algorithm \ref{alg:generalization} for the pseudocode of an algorithm witnessing the theorem statement. 

\begin{algorithm}[H]
\caption{Implementation of a generalization algorithm given an implementation of a filtering algorithm\label{alg:generalization}}
\begin{algorithmic}[1]
    \STATE \textbf{Input}: Training set $S = \Sclean \cup \Sbackdoor$\\ satisfying $\abs{\Sclean} = \Omega\inparen{\epsclean^{-2}\inparen{\vc{\robustloss} + \logv{\nfrac{1}{\delta}}}}$
    \STATE Run the filtering algorithm on $S$ to obtain $S'$ satisfying the conditions in the theorem statement
    \STATE \textbf{Output}: Output $\hhat$, defined as $\hhat \coloneqq \argmin_{h\in\cH} \robustloss(h, S')$
\end{algorithmic}
\end{algorithm}

Recall that we have drawn enough samples to achieve uniform convergence (see \cite{Cullina2018-os} and \cite{Montasser2020-ir}); in particular, assuming that our previous steps succeeded in removing very few points from $\Sclean$, then for all $h \in \cH$, we have with probability $1-\delta$:
$$\abs{\robustloss(h, \cD) - \robustloss(h, \Sclean)} \le \epsclean$$
Observe that we have deleted at most $m\cdot 2\epsclean$ points from $\Sclean$. Let $\Sclean' \coloneqq S' \cap \Sclean$ (i.e., the surviving members of $\Sclean$ from our filtering procedure). We start with the following claim.
\begin{claim}
The following holds for all $h \in \cH$:
$$\abs{\robustloss(h, \Sclean) - \robustloss(h, \Sclean')} \le \epsclean$$
\end{claim}
\begin{proof}
\pfFractionSubtraction

Now, let $a$ denote the number of samples from $\Sclean$ that $h$ incurs robust loss on, let $b$ be the total number of samples from $\Sclean$, and let $c$ be the number of samples our filtering procedure deletes from $\Sclean$. It is easy to see that $\nfrac{a}{b}$ corresponds $\robustloss(h, \Sclean)$ and that $\robustloss(h, \Sclean') \in \insquare{\max\inbraces{0,\nfrac{(a-c)}{(b-c)}}, \min\inbraces{1,\nfrac{a}{(b-c)}}}$. From our argument above, this means that we must have:
$$\robustloss(h, \Sclean') \in \insquare{\robustloss(h, \Sclean) \pm \frac{\epsclean(1-\alpha) m}{(1-\alpha)m}}$$
Finally:
$$\frac{\epsclean(1-\alpha) m}{(1-\alpha)m} = \epsclean$$
and we're done.
\end{proof}
We now use our claim and triangle inequality to write:
\begin{align*}
    \abs{\robustloss(h, \Sclean') - \robustloss(h, \cD)} \le &\abs{\robustloss(h, \Sclean) - \robustloss(h, \Sclean')} + \\
    &\abs{\robustloss(h, \cD) - \robustloss(h, \Sclean)}\\
    \le &\epsclean
\end{align*}
Next, consider some $\hhat$ satisfying $\robustloss(\hhat, S') \lesssim \epsclean$ (which must exist, as per our argument in Part 3), and observe that, for a constant $C$:
\begin{align*}
    \robustloss(\hhat, S') &\ge (1-C\epsclean)\robustloss(\hhat, S' \cap \Sclean) + C\epsclean\robustloss(\hhat, S' \cap \Sbackdoor) \\
    &\ge (1-C\epsclean)\robustloss(\hhat, \Sclean') \\
    \implies \robustloss(\hhat, \Sclean') &\le \frac{\epsclean}{1-C\epsclean} = 2\epsclean\inparen{\frac{1}{1-C\epsclean}} \lesssim \epsclean
\end{align*}
We now use the fact that $\abs{\robustloss(h, \Sclean') - \robustloss(h, \cD)} \le \epsclean$ to arrive at the conclusion that $\robustloss(h,\cD) \lesssim \epsclean$, which is what we wanted to show.
\end{proof}}
See Algorithm \ref{alg:generalization} in the Appendix for the pseudocode of an algorithm witnessing the theorem statement. 

\newcommand{\thmGeneralizeToFilter}{Set $\epsclean \le \nfrac{1}{10}$ and $\alpha \le \nfrac{1}{6}$. 

If there exists an algorithm that, given at most a $2\alpha$ fraction of outliers in the training set, can output a hypothesis satisfying $\robustloss(\hhat, \cD) \le \epsclean$ with probability $1-\delta$ over the draw of the training set, then there exists an algorithm that given a training set $S = \Sclean \cup \Sbackdoor$ satisfying $\abs{\Sclean} \ge \Omega\inparen{\epsclean^{-2}\inparen{\vc{\robustloss} + \logv{\nfrac{1}{\delta}}}}$ outputs a subset $S' \subseteq S$ with the property that $\robustloss\inparen{\argmin_{h \in \cH} \robustloss\inparen{h, S'}, \cD} \lesssim \epsclean$ with probability $1-7\delta$.
}
\begin{theorem}[Generalization Implies Filtering (Appendix Theorem \ref{app:generalize_to_filter})]
\label{thm:generalize_to_filter}
\thmGeneralizeToFilter
\end{theorem}
\newcommand{\pfGeneralizeToFilter}{\begin{proof}
See Algorithm \ref{alg:filtering} for the pseudocode of an algorithm witnessing the theorem statement. 

At a high level, our proof proceeds as follows. We first show that the partitioning step results in partitions that don't have too high of a fraction of outliers, which will allow us to call the filtering procedure without exceeding the outlier tolerance. Then, we will show that the hypotheses $\hhat_L$ and $\hhat_R$ mark most of the backdoor points for deletion while marking only few of the clean points for deletion. Finally, we will show that although $\hhat$ is learned on $S'$ that is not sampled i.i.d from $\cD$, $\hhat$ still generalizes to $\cD$ without great decrease in accuracy.

\begin{algorithm}[H]
\caption{Implementation of a filtering algorithm given an implementation of a generalization algorithm\label{alg:filtering}}
\begin{algorithmic}[1]
    \STATE \textbf{Input}: Training set $S = \Sclean \cup \Sbackdoor$\\ satisfying $\abs{\Sclean} = \Omega\inparen{\epsclean^{-2}\inparen{\vc{\robustloss} + \logv{\nfrac{1}{\delta}}}}$
    \STATE Calculate $\hhat = \argmin_{h \in \cH} \robustloss(h, S)$ and early-return $S$ if $\robustloss(\hhat, S) \le C\epsclean$, for some universal constant $C$
    \STATE Randomly partition $S$ into two equal halves $S_L$ and $S_R$
    \STATE Run the generalizing algorithm to obtain $\hhat_L$ and $\hhat_R$ using training sets $S_L$ and $S_R$, respectively
    \STATE Run $\hhat_L$ on $S_R$ and mark every mistake that $\hhat_L$ makes on $S_R$, and similarly for $\hhat_R$
    \STATE Remove all marked examples to obtain a new training set $S' \subseteq S$
    \STATE \textbf{Output}: $S'$ such that $\hhat = \argmin_{h \in \cH} \robustloss(h, S')$ satisfies $\robustloss(\hhat, \cD) \lesssim \epsclean$ with probability $1-\delta$
\end{algorithmic}
\end{algorithm}

We have two cases to consider based on the number of outliers in our training set. Let $m$ be the total number of examples in our training set.

\paragraph{Case 1 -- $\alpha m \le \max\inbraces{\nfrac{2}{3\epsclean} \cdot \logv{\nfrac{1}{\delta}}, 24\logv{\nfrac{2}{\delta}}}$}

It is easy to see that $\cL(\hstar, S) \le \alpha$. Using this, we can write:
\begin{align*}
    \cL(\hstar, S) &\le \alpha \\
    &\frac{2}{3\epsclean \cdot m} \cdot \logv{\frac{1}{\delta}} \\
    &\lesssim \frac{\epsclean}{\vc{\cH} + \logv{\nfrac{1}{\delta}}} \cdot \logv{\frac{1}{\delta}} \\
    &< \epsclean
\end{align*}
which implies that we exit the routine via the early-return. From uniform convergence, this implies that with probability $1-\delta$ over the draws of $S$, we have $\robustloss\inparen{\argmin_{h \in \cH} \robustloss\inparen{h, S'}, \cD} \lesssim \epsclean$.

In the other case, we write:
\begin{align*}
    \cL(\hstar, S) &\le \alpha \\
    &\le \frac{24\logv{\nfrac{2}{\delta}}}{m} \\
    &\lesssim \frac{\epsclean^2\logv{\nfrac{1}{\delta}}}{\vc{\cH} + \logv{\nfrac{1}{\delta}}} \\
    &\lesssim \epsclean^2 \le \epsclean
\end{align*}
and the rest follows from a similar argument.

\paragraph{Case 2 -- $\alpha m \ge \max\inbraces{\nfrac{2}{3\epsclean} \cdot \logv{\nfrac{1}{\delta}}, 24\logv{\nfrac{2}{\delta}}}$}

Let $\tau = \delta$; we make this rewrite to help simplify the various failure events.

\paragraph{Part 1 -- Partitioning Doesn't Affect Outlier Balance} Define indicator random variables $X_i$ such that $X_i$ is $1$ if and only if example $i$ ends up in $S_R$. We want to show that:
$$\prv{\sum_{i \in \Sbackdoor} X_i \notin \insquare{0.5, 1.5}\alpha\cdot\nfrac{m}{2}} \le \tau$$
Although the $X_i$ are not independent, they are negatively associated, so we can still use the Chernoff Bound and the fact that the number of outliers $\alpha m \ge 24\logv{\nfrac{2}{\tau}}$:
\begin{align*}
    \prv{\sum_{i \in \Sbackdoor} X_i \notin \insquare{0.5, 1.5}\alpha\cdot\nfrac{m}{2}} &\le 2\expv{-\frac{\nfrac{\alpha}{2} \cdot m \cdot \nfrac{1}{4}}{3}} \\
    &\le 2\expv{-\frac{\alpha m}{24}} \le \tau
\end{align*}
Moreover, if $S_L$ has a $\insquare{\nfrac{\alpha}{2}, \nfrac{3\alpha}{2}}$ fraction of outliers, then it also follows that $S_R$ has a $\insquare{\nfrac{\alpha}{2}, \nfrac{3\alpha}{2}}$ fraction of outliers. Thus, this step succeeds with probability $1-\tau$.

\paragraph{Part 2 -- Approximately Correctly Marking Points} We now move onto showing that $\hhat_L$ deletes most outliers from $S_R$ while deleting few clean points. Recall that $\hhat_L$ satisfies $\robustloss(\hhat_L, \cD) \le \epsclean$ with probability $1-\delta$. Thus, we have that $\hhat_L$ labels the outliers as opposite the target label with probability at least $1-\epsclean$. Since we have that the number of outliers $\alpha m \ge \nfrac{2}{3\epsclean} \cdot \logv{\nfrac{1}{\tau}}$, we have from Chernoff Bound (let $X_i$ be the indicator random variable that is $1$ when $\hhat_L$ classifies a backdoored example as the target label):
\begin{align*}
    \prv{\sum_{i \in \Sbackdoor \cap S_R} X_i \ge 2 \cdot \inparen{\epsclean \cdot \frac{3}{2}\alpha m}} \le \expv{-\epsclean \cdot \frac{3}{2}\alpha m} \le \tau
\end{align*}
Thus, with probability $1-2\tau$, we mark all but at most $\epsclean \cdot 6\alpha m$ outliers across both $S_R$ and $S_L$; since we impose that $\alpha \lesssim 1$, we have that we delete all but a $c\epsclean$ fraction of outliers for some universal constant $c$.

It remains to show that we don't delete too many good points. Since $\hhat_L$ has true error at most $\epsclean$ and using the fact that $m(1-\nfrac{\alpha}{2}) \ge m(1-\alpha) \ge m\alpha \ge \frac{2\logv{\nfrac{1}{\tau}}}{\epsclean}$, from the Chernoff Bound, we have (let $X_i$ be the indicator random variable that is $1$ when $\hhat_L$ misclassifies a clean example):
\begin{align*}
    \prv{\sum_{i \in \Sclean \cap S_R} X_i \ge 2\cdot\inparen{\epsclean \cdot \inparen{1-\nfrac{\alpha}{2}} \cdot \frac{m}{2}}} \le \expv{-\epsclean \cdot \inparen{1-\nfrac{\alpha}{2}} \cdot \frac{m}{2}} \le \tau
\end{align*}
From a union bound over the runs of $\hhat_L$ and $\hhat_R$, we have that with probability $1-2\tau$, we mark at most $2m\epsclean\cdot\inparen{1-\nfrac{\alpha}{2}} \le 2m\epsclean$ clean points for deletion. From a union bound, we have that this whole step succeeds with probability $1-4\tau-2\delta$.

\paragraph{Part 3 -- There Exists a Low-Error Classifier} At this stage, we have a training set $S'$ that has at least $m(1-2\epsclean)$ clean points and at most $\epsclean \cdot 6\alpha m$ outliers. Recall that $\hstar$ incurs robust loss on none of the clean points and incurs robust loss on every outlier. This implies that $\hstar$ has empirical robust loss at most:
\begin{align*}
    \frac{\epsclean \cdot 6\alpha m}{m(1-2\epsclean)} = \frac{6\alpha\epsclean}{1-2\epsclean} \le 2\epsclean
\end{align*}
where we use the fact that we pick $\epsclean \le \nfrac{1}{10} < \nfrac{1}{4}$ and $\alpha \le \nfrac{1}{6}$. From this, it follows that $\hhat = \argmin_{h \in \cH} \robustloss(h, S')$ satisfies $\robustloss(\hhat, S') \le 2\epsclean$.

\paragraph{Part 4 -- Generalizing from $S'$ to $\cD$} We now have to argue that $\robustloss(\hhat, S') \le 2\epsclean$ implies $\robustloss(\hhat, \cD) \lesssim \epsclean$. Recall that we have drawn enough samples to achieve uniform convergence (see \cite{Cullina2018-os} and \cite{Montasser2020-ir}); in particular, assuming that our previous steps succeeded in removing very few points from $\Sclean$, then for all $h \in \cH$, we have with probability $1-\delta$:
$$\abs{\robustloss(h, \cD) - \robustloss(h, \Sclean)} \le \epsclean$$
Observe that we have deleted at most $m\cdot 2\epsclean$ points from $\Sclean$. Let $\Sclean' \coloneqq S' \cap \Sclean$ (i.e., the surviving members of $\Sclean$ from our filtering procedure). We start with the following claim.
\begin{claim}
The following holds for all $h \in \cH$:
$$\abs{\robustloss(h, \Sclean) - \robustloss(h, \Sclean')} < 3\epsclean$$
\end{claim}
\begin{proof}
Recall that in the proof of Theorem \ref{thm:filter_to_generalize}, we showed that for positive numbers $a, b, c$ we have:
$$\max\inbraces{0, \frac{a-c}{b-c}}, \min\inbraces{1, \frac{a}{b-c}} \in \insquare{\frac{a}{b} \pm \frac{c}{b}}$$
Now, let $a$ denote the number of samples from $\Sclean$ that $h$ incurs robust loss on, let $b$ be the total number of samples from $\Sclean$, and let $c$ be the number of samples our filtering procedure deletes from $\Sclean$. It is easy to see that $\nfrac{a}{b}$ corresponds $\robustloss(h, \Sclean)$ and that $\robustloss(h, \Sclean') \in \insquare{\max\inbraces{0,\nfrac{(a-c)}{(b-c)}}, \min\inbraces{1,\nfrac{a}{(b-c)}}}$. From our argument above, this means that we must have:
$$\robustloss(h, \Sclean') \in \insquare{\robustloss(h, \Sclean) \pm \frac{2\epsclean m}{(1-\alpha)m}}$$
Finally:
$$\frac{2\epsclean m}{(1-\alpha)m} = \frac{2\epsclean}{(1-\alpha)} \le \frac{2\epsclean}{\nfrac{5}{6}} < 3\epsclean$$
and we're done.
\end{proof}
We now use our claim and triangle inequality to write:
\begin{align*}
    \abs{\robustloss(h, \Sclean') - \robustloss(h, \cD)} \le &\abs{\robustloss(h, \Sclean) - \robustloss(h, \Sclean')} + \\
    &\abs{\robustloss(h, \cD) - \robustloss(h, \Sclean)}\\
    < &4\epsclean
\end{align*}
Next, consider some $\hhat$ satisfying $\robustloss(\hhat, S') \le 2\epsclean$ (which must exist, as per our argument in Part 3), and observe that:
\begin{align*}
    \robustloss(\hhat, S') &\ge (1-2\epsclean)\robustloss(\hhat, S' \cap \Sclean) + 2\epsclean\robustloss(\hhat, S' \cap \Sbackdoor) \\
    &\ge (1-2\epsclean)\robustloss(\hhat, \Sclean') \\
    \implies \robustloss(\hhat, \Sclean') &\le \frac{2\epsclean}{1-2\epsclean} = 2\epsclean\inparen{\frac{1}{1-2\epsclean}} \le \frac{5\epsclean}{2}
\end{align*}
We now use the fact that $\abs{\robustloss(h, \Sclean') - \robustloss(h, \cD)} < 4\epsclean$ to arrive at the conclusion that $\robustloss(h,\cD) < \nfrac{13}{2}\cdot\epsclean$, which is what we wanted to show.

The constants in the statement of Theorem \ref{thm:generalize_to_filter} follow from setting $\tau = \delta$.
\end{proof}}

See Algorithm \ref{alg:filtering} in the Appendix for the pseudocode of an algorithm witnessing Theorem \ref{thm:generalize_to_filter}. Note that there is a factor-$2$ separation between the values of $\alpha$ used in the filtering and generalizing routines above; this is a limitation of our current analysis.

The upshot of Theorems \ref{thm:filter_to_generalize} and \ref{thm:generalize_to_filter} is that in order to obtain a classifier robust to backdoor perturbations at test-time, it is statistically necessary and sufficient to design an algorithm that can filter sufficiently many outliers to where directly minimizing the robust loss (e.g., adversarial training) yields a generalizing classifier. Furthermore, computational equivalence holds in the case where minimizing the robust loss on the training set can be done efficiently (such as in the case of linear separators with closed, convex, bounded, origin-symmetric perturbation sets -- see \cite{Montasser2020-ir}). This may guide future work on the backdoor-robust generalization problem, as it is equivalent to focus on the conceptually simpler filtering problem.

\section{Related Works}
\label{sec:related_works}

Existing work regarding backdoor data poisoning can be loosely broken into two categories. For a more general survey of backdoor attacks, please see the work of \cite{Li2020-my}.

\paragraph{Attacks} To the best of our knowledge, the first work to empirically demonstrate the existence of backdoor poisoning attacks is that of \cite{Gu2017-rj}. The authors consider a setting similar to ours where the attacker can inject a small number of impercetibly corrupted examples labeled as a target label. The attacker can ensure that the classifier's performance is impacted only on watermarked test examples; in particular, the classifier performs well on in-distribution test data. Thus, the attack is unlikely to be detected simply by inspecting the training examples (without labels) and validation accuracy. The work of \cite{Chen2017-kq} and \cite{Gu2019-ip} explores a similar setting.

The work of \cite{Wang2020-yt} discusses theoretical aspects of backdoor poisoning attacks in a federated learning scenario. Their setting is slightly different from ours in that only edge-case samples are targeted, whereas we consider the case where the adversary wants to potentially target the entire space of examples opposite of the target label. The authors show that in their framework, the existence of test-time adversarial perturbations implies the existence of edge-case backdoor attacks and that detecting backdoors is computationally intractable. 

Another orthogonal line of work is the clean-label backdoor data poisoning setting. Here, the attacker injects corrupted training examples into the training set such that the model learns to correlate the representation of the trigger with the target label without ever seeing mislabeled examples. The work of \cite{Saha2019-ce} and \cite{Turner2019-jc} give empirically successful constructions of such an attack. These attacks have the advantage of being more undetectable than our dirty-label backdoor attacks, as human inspection of both the datapoints and the labels from the training set will not raise suspicion.

Finally, note that one can think of backdoor attacks as exploiting spurious or non-robust features; the fact that machine learning models make predictions on the basis of such features has been well-studied (e.g. see \cite{Ribeiro2016-ig}, \cite{Ilyas2019-ot}, \cite{Xiao2020-ju}).

\paragraph{Defenses} Although there are a variety of empirical defenses against backdoor attacks with varying success rates, we draw attention to two defenses that are theoretically motivated and that most closely apply to the setting we consider in our work.

As far as we are aware, one of the first theoretically motivated defenses against backdoor poisoning attacks involves using \emph{spectral signatures}. Spectral signatures (\cite{Tran2018-bf}) relies on the fact that outliers necessarily corrupt higher-order moments of the empirical distribution, especially in the feature space. Thus, to find outliers, one can estimate class means and covariances and filter the points most correlated with high-variance projections of the empirical distribution in the feature space. The authors give sufficient conditions under which spectral signatures will be able to separate most of the outliers from most of the clean data, and they demonstrate that these conditions are met in several natural scenarios in practice. 

Another defense with some provable backing is \emph{Iterative Trimmed Loss Minimization} (ITLM), which was first used against backdoor attacks by \cite{Shen2018-jx}. ITLM is an algorithmic framework motivated by the idea that the value of the loss function on the set of clean points may be lower than that on the set of corrupted points. Thus, an ITLM-based procedure selects a low-loss subset of the training data and performs a model update step on this subset. This alternating minimization is repeated until the model loss is sufficiently small. The heuristic behind ITLM holds in practice, as per the evaluations from \cite{Shen2018-jx}.

\paragraph{Memorization of Training Data} The work of \cite{Arpit2017-yz} and \cite{Feldman2020-ex} discuss the ability of neural networks to memorize their training data. Specifically, the work of \cite{Arpit2017-yz} empirically discusses how memorization plays into the learning dynamics of neural networks via fitting random labels. The work of \cite{Feldman2020-ex} experimentally validates the ``long tail theory'', which posits that data distributions in practice tend to have a large fraction of their mass allocated to ``atypical'' examples; thus, the memorization of these rare examples is actually necessary for generalization.

Our notion of memorization is different in that we consider excess capacity \emph{on top of the learning problem at hand}. In other words, we require that there exist a classifier in the hypothesis class that behaves correctly on on-distribution data in addition to memorizing specially curated off-distribution data. 


\section{Conclusions and Future Work}
\label{sec:conclusion}

\paragraph{Conclusions} We gave a framework under which backdoor data poisoning attacks can be studied. We then showed that, under this framework, a formal notion of excess capacity present in the learning problem is necessary and sufficient for the existence of a backdoor attack. Finally, in the algorithmic setting, we showed that under certain assumptions, adversarial training can detect the presence of backdoors and that filtering backdoors from a training set is equivalent to learning a backdoor-robust classifier.

\paragraph{Future Work} There are several interesting problems directly connected to our work for which progress would yield a better understanding of backdoor attacks. Perhaps the most important is to find problems for which there simultaneously exist efficient backdoor filtering algorithms and efficient adversarial training algorithms. It would also be illuminating to determine the extent to which adversarial training detects backdoor attacks in deep learning\footnote{This scenario is not automatically covered by our result since the generalization of neural networks to unseen data cannot be explained by traditional uniform convergence and VC-dimension arguments (see \cite{Neyshabur2017-mk} and \cite{Zhang2016-mh}).}. Finally, we believe that our notion of memorization capacity can find applications beyond the scope of this work. It would be particularly interesting to see if memorization capacity has applications to explaining robustness or lack thereof to test-time adversarial perturbations.

\paragraph{Societal Impacts} Defenses against backdoor attacks may impede the functionality of several privacy-preserving applications. Most notably, the Fawkes system (see \cite{Shan2020-dr}) relies on a backdoor data poisoning attack to preserve its users' privacy, and such a system could be compromised if it were known how to reliably defend against backdoor data poisoning attacks in such a setting.

\paragraph{Acknowledgments} This work was supported in part by the National Science Foundation under grant CCF-1815011 and by the Defense Advanced Research Projects Agency under cooperative agreement HR00112020003. The views expressed in this work do not necessarily reflect the position or the policy of the Government and no official endorsement should be inferred. Approved for public release; distribution is unlimited.

NSM thanks Surbhi Goel for suggesting the experiments run in the paper.

\nocite{*}
\printbibliography

@ARTICLE{Zhang2016-mh,
  title         = "Understanding deep learning requires rethinking
                   generalization",
  author        = "Zhang, Chiyuan and Bengio, Samy and Hardt, Moritz and Recht,
                   Benjamin and Vinyals, Oriol",
  month         =  nov,
  year          =  2016,
  archivePrefix = "arXiv",
  primaryClass  = "cs.LG",
  eprint        = "1611.03530"
}

@INPROCEEDINGS{Borgnia2021-vk,
  title     = "Strong Data Augmentation Sanitizes Poisoning and Backdoor
               Attacks Without an Accuracy Tradeoff",
  booktitle = "{ICASSP} 2021 - 2021 {IEEE} International Conference on
               Acoustics, Speech and Signal Processing ({ICASSP})",
  author    = "Borgnia, Eitan and Cherepanova, Valeriia and Fowl, Liam and
               Ghiasi, Amin and Geiping, Jonas and Goldblum, Micah and
               Goldstein, Tom and Gupta, Arjun",
  pages     = "3855--3859",
  month     =  jun,
  year      =  2021
}

@ARTICLE{Feldman2020-ex,
  title         = "What Neural Networks Memorize and Why: Discovering the Long
                   Tail via Influence Estimation",
  author        = "Feldman, Vitaly and Zhang, Chiyuan",
  month         =  aug,
  year          =  2020,
  archivePrefix = "arXiv",
  primaryClass  = "cs.LG",
  eprint        = "2008.03703"
}

@INPROCEEDINGS{Arpit2017-yz,
  title     = "A Closer Look at Memorization in Deep Networks",
  booktitle = "Proceedings of the 34th International Conference on Machine
               Learning",
  author    = "Arpit, Devansh and Jastrz{\k e}bski, Stanis{\l}aw and Ballas,
               Nicolas and Krueger, David and Bengio, Emmanuel and Kanwal,
               Maxinder S and Maharaj, Tegan and Fischer, Asja and Courville,
               Aaron and Bengio, Yoshua and Lacoste-Julien, Simon",
  editor    = "Precup, Doina and Teh, Yee Whye",
  publisher = "PMLR",
  volume    =  70,
  pages     = "233--242",
  series    = "Proceedings of Machine Learning Research",
  year      =  2017
}

@ARTICLE{Gu2017-rj,
  title         = "{BadNets}: Identifying Vulnerabilities in the Machine
                   Learning Model Supply Chain",
  author        = "Gu, Tianyu and Dolan-Gavitt, Brendan and Garg, Siddharth",
  month         =  aug,
  year          =  2017,
  archivePrefix = "arXiv",
  primaryClass  = "cs.CR",
  eprint        = "1708.06733"
}

@ARTICLE{Daniely2012-pw,
  title         = "Multiclass Learning Approaches: A Theoretical Comparison
                   with Implications",
  author        = "Daniely, Amit and Sabato, Sivan and Shwartz, Shai Shalev",
  month         =  may,
  year          =  2012,
  archivePrefix = "arXiv",
  primaryClass  = "cs.LG",
  eprint        = "1205.6432"
}

@article{lecun-mnisthandwrittendigit-2010,
  author = {LeCun, Yann and Cortes, Corinna},
  howpublished = {http://yann.lecun.com/exdb/mnist/},
  lastchecked = {2016-01-14 14:24:11},
  title = {{MNIST} handwritten digit database},
  url = {http://yann.lecun.com/exdb/mnist/},
  username = {mhwombat},
  year = 2010
}

@INPROCEEDINGS{Ribeiro2016-ig,
  title     = "``Why Should {I} Trust You?'': Explaining the Predictions of Any
               Classifier",
  booktitle = "Proceedings of the 22nd {ACM} {SIGKDD} International Conference
               on Knowledge Discovery and Data Mining",
  author    = "Ribeiro, Marco Tulio and Singh, Sameer and Guestrin, Carlos",
  publisher = "Association for Computing Machinery",
  pages     = "1135--1144",
  series    = "KDD '16",
  month     =  aug,
  year      =  2016,
  address   = "New York, NY, USA",
  location  = "San Francisco, California, USA"
}

@inproceedings{
Xiao2020-ju,
title={Noise or Signal: The Role of Image Backgrounds in Object Recognition},
author={Kai Yuanqing Xiao and Logan Engstrom and Andrew Ilyas and Aleksander Madry},
booktitle={International Conference on Learning Representations},
year={2021},
url={https://openreview.net/forum?id=gl3D-xY7wLq}
}

@INPROCEEDINGS{Ilyas2019-ot,
  title     = "Adversarial Examples Are Not Bugs, They Are Features",
  booktitle = "Advances in Neural Information Processing Systems",
  author    = "Ilyas, Andrew and Santurkar, Shibani and Tsipras, Dimitris and
               Engstrom, Logan and Tran, Brandon and Madry, Aleksander",
  editor    = "Wallach, H and Larochelle, H and Beygelzimer, A and
               d' Alch{\'e}-Buc, F and Fox, E and
               Garnett, R",
  publisher = "Curran Associates, Inc.",
  volume    =  32,
  year      =  2019
}

@BOOK{Boyd2004-ey,
  title     = "Convex Optimization",
  author    = "Boyd, Stephen and Vandenberghe, Lieven",
  publisher = "Cambridge University Press",
  month     =  mar,
  year      =  2004
}

@INPROCEEDINGS{Cullina2018-os,
  title     = "{PAC-learning} in the presence of adversaries",
  booktitle = "Advances in Neural Information Processing Systems",
  author    = "Cullina, Daniel and Bhagoji, Arjun Nitin and Mittal, Prateek",
  editor    = "Bengio, S and Wallach, H and Larochelle, H and Grauman, K and
               Cesa-Bianchi, N and Garnett, R",
  publisher = "Curran Associates, Inc.",
  volume    =  31,
  year      =  2018
}

@INPROCEEDINGS{Neyshabur2017-mk,
  title     = "Exploring Generalization in Deep Learning",
  booktitle = "Advances in Neural Information Processing Systems",
  author    = "Neyshabur, Behnam and Bhojanapalli, Srinadh and Mcallester,
               David and Srebro, Nati",
  editor    = "Guyon, I and Luxburg, U V and Bengio, S and Wallach, H and
               Fergus, R and Vishwanathan, S and Garnett, R",
  publisher = "Curran Associates, Inc.",
  volume    =  30,
  year      =  2017
}

@INPROCEEDINGS{Montasser2020-ir,
  title     = "Efficiently Learning Adversarially Robust Halfspaces with Noise",
  booktitle = "Proceedings of the 37th International Conference on Machine
               Learning",
  author    = "Montasser, Omar and Goel, Surbhi and Diakonikolas, Ilias and
               Srebro, Nathan",
  editor    = "Iii, Hal Daum{\'e} and Singh, Aarti",
  publisher = "PMLR",
  volume    =  119,
  pages     = "7010--7021",
  series    = "Proceedings of Machine Learning Research",
  year      =  2020
}

@misc{Turner2019-jc,
  title         = "{Label-Consistent} Backdoor Attacks",
  author        = "Turner, Alexander and Tsipras, Dimitris and Madry,
                   Aleksander",
  month         =  dec,
  year          =  2019,
  archivePrefix = "arXiv",
  primaryClass  = "stat.ML",
  eprint        = "1912.02771"
}

@ARTICLE{Gu2019-ip,
  title    = "{BadNets}: Evaluating Backdooring Attacks on Deep Neural Networks",
  author   = "Gu, Tianyu and Liu, Kang and Dolan-Gavitt, Brendan and Garg,
              Siddharth",
  journal  = "IEEE Access",
  volume   =  7,
  pages    = "47230--47244",
  year     =  2019
}

@misc{Li2020-my,
  title         = "Backdoor Learning: A Survey",
  author        = "Li, Yiming and Wu, Baoyuan and Jiang, Yong and Li, Zhifeng
                   and Xia, Shu-Tao",
  month         =  jul,
  year          =  2020,
  archivePrefix = "arXiv",
  primaryClass  = "cs.CR",
  eprint        = "2007.08745"
}

@inproceedings{Shan2020-dr,
  title     = "Fawkes: Protecting privacy against unauthorized deep learning
               models",
  author    = "Shan, S and Wenger, E and Zhang, J and Li, H and Zheng, H and
               {others}",
  booktitle = "Proceedings of the {Twenty-Ninth} USENIX Security Symposium",
  journal   = "29th \{USENIX\} Security",
  publisher = "usenix.org",
  year      =  2020
}

@misc{Koh2018-bz,
  title         = "Stronger Data Poisoning Attacks Break Data Sanitization
                   Defenses",
  author        = "Koh, Pang Wei and Steinhardt, Jacob and Liang, Percy",
  month         =  nov,
  year          =  2018,
  archivePrefix = "arXiv",
  primaryClass  = "stat.ML",
  eprint        = "1811.00741"
}

@misc{Truong2020-dk,
  title         = "Systematic Evaluation of Backdoor Data Poisoning Attacks on
                   Image Classifiers",
  author        = "Truong, Loc and Jones, Chace and Hutchinson, Brian and
                   August, Andrew and Praggastis, Brenda and Jasper, Robert and
                   Nichols, Nicole and Tuor, Aaron",
  month         =  apr,
  year          =  2020,
  archivePrefix = "arXiv",
  primaryClass  = "cs.CV",
  eprint        = "2004.11514"
}

@INPROCEEDINGS{Adi2018-fz,
  title     = "Turning your weakness into a strength: watermarking deep neural
               networks by backdooring",
  booktitle = "Proceedings of the 27th {USENIX} Conference on Security
               Symposium",
  author    = "Adi, Yossi and Baum, Carsten and Cisse, Moustapha and Pinkas,
               Benny and Keshet, Joseph",
  publisher = "USENIX Association",
  pages     = "1615--1631",
  series    = "SEC'18",
  month     =  aug,
  year      =  2018,
  address   = "USA",
  location  = "Baltimore, MD, USA"
}

@misc{Chen2017-kq,
      title={Targeted Backdoor Attacks on Deep Learning Systems Using Data Poisoning}, 
      author={Xinyun Chen and Chang Liu and Bo Li and Kimberly Lu and Dawn Song},
      year={2017},
      eprint={1712.05526},
      archivePrefix={arXiv},
      primaryClass={cs.CR}
}

@inproceedings{Wang2020-yt,
 author = {Wang, Hongyi and Sreenivasan, Kartik and Rajput, Shashank and Vishwakarma, Harit and Agarwal, Saurabh and Sohn, Jy-yong and Lee, Kangwook and Papailiopoulos, Dimitris},
 booktitle = {Advances in Neural Information Processing Systems},
 editor = {H. Larochelle and M. Ranzato and R. Hadsell and M. F. Balcan and H. Lin},
 pages = {16070--16084},
 publisher = {Curran Associates, Inc.},
 title = {Attack of the Tails: Yes, You Really Can Backdoor Federated Learning},
 url = {https://proceedings.neurips.cc/paper/2020/file/b8ffa41d4e492f0fad2f13e29e1762eb-Paper.pdf},
 volume = {33},
 year = {2020}
}

@misc{Madry2017-ep,
  title         = "Towards Deep Learning Models Resistant to Adversarial
                   Attacks",
  author        = "Madry, Aleksander and Makelov, Aleksandar and Schmidt,
                   Ludwig and Tsipras, Dimitris and Vladu, Adrian",
  month         =  jun,
  year          =  2017,
  archivePrefix = "arXiv",
  primaryClass  = "stat.ML",
  eprint        = "1706.06083"
}

@INPROCEEDINGS{Montasser2019-ro,
  title     = "{VC} Classes are Adversarially Robustly Learnable, but Only
               Improperly",
  booktitle = "Proceedings of the {Thirty-Second} Conference on Learning Theory",
  author    = "Montasser, Omar and Hanneke, Steve and Srebro, Nathan",
  editor    = "Beygelzimer, Alina and Hsu, Daniel",
  publisher = "PMLR",
  volume    =  99,
  pages     = "2512--2530",
  series    = "Proceedings of Machine Learning Research",
  year      =  2019,
  address   = "Phoenix, USA"
}

@ARTICLE{Saha2019-ce,
  title    = "Hidden Trigger Backdoor Attacks",
  author   = "Saha, Aniruddha and Subramanya, Akshayvarun and Pirsiavash, Hamed",
  journal  = "AAAI",
  volume   =  34,
  number   =  07,
  pages    = "11957--11965",
  month    =  apr,
  year     =  2020,
  language = "en"
}

@INPROCEEDINGS{Shen2018-jx,
  title     = "Learning with Bad Training Data via Iterative Trimmed Loss
               Minimization",
  booktitle = "Proceedings of the 36th International Conference on Machine
               Learning",
  author    = "Shen, Yanyao and Sanghavi, Sujay",
  editor    = "Chaudhuri, Kamalika and Salakhutdinov, Ruslan",
  publisher = "PMLR",
  volume    =  97,
  pages     = "5739--5748",
  series    = "Proceedings of Machine Learning Research",
  year      =  2019
}

@misc{Shafahi2018-ns,
  title         = "Poison Frogs! Targeted {Clean-Label} Poisoning Attacks on
                   Neural Networks",
  author        = "Shafahi, Ali and Ronny Huang, W and Najibi, Mahyar and
                   Suciu, Octavian and Studer, Christoph and Dumitras, Tudor
                   and Goldstein, Tom",
  month         =  apr,
  year          =  2018,
  archivePrefix = "arXiv",
  primaryClass  = "cs.LG",
  eprint        = "1804.00792"
}

@BOOK{Shalev-Shwartz2014-oj,
  title     = "Understanding Machine Learning: From Theory to Algorithms",
  author    = "Shalev-Shwartz, Shai and Ben-David, Shai",
  publisher = "Cambridge University Press",
  month     =  may,
  year      =  2014,
  language  = "en"
}

@BOOK{Vershynin2018-xn,
  title     = "{High-Dimensional} Probability: An Introduction with
               Applications in Data Science",
  author    = "Vershynin, Roman",
  publisher = "Cambridge University Press",
  month     =  sep,
  year      =  2018,
  language  = "en"
}

@misc{Hardt2012-xk,
  title         = "Algorithms and Hardness for Robust Subspace Recovery",
  author        = "Hardt, Moritz and Moitra, Ankur",
  month         =  nov,
  year          =  2012,
  archivePrefix = "arXiv",
  primaryClass  = "cs.CC",
  eprint        = "1211.1041"
}

@INPROCEEDINGS{Tran2018-bf,
  title     = "Spectral Signatures in Backdoor Attacks",
  booktitle = "Advances in Neural Information Processing Systems",
  author    = "Tran, Brandon and Li, Jerry and Madry, Aleksander",
  editor    = "Bengio, S and Wallach, H and Larochelle, H and Grauman, K and
               Cesa-Bianchi, N and Garnett, R",
  publisher = "Curran Associates, Inc.",
  volume    =  31,
  year      =  2018
}

\newpage

\appendix

\section{Restatement of Theorems and Full Proofs}
\label{sec:app_proof}

In this section, we will restate our main results and give full proofs.

\begin{theorem}[Existence of Backdoor Data Poisoning Attack (Theorem \ref{thm:existence_linear_backdoor})]
\label{app:existence_linear_backdoor}
\thmExistenceLinearBackdoor
\end{theorem}
\pfExistenceLinearBackdoor

\begin{corollary}[Overparameterized Linear Classifier (Corollary \ref{cor:linear_backdoor})]
\label{app:linear_backdoor}
\corLinearBackdoor
\end{corollary}
\pfCorLinearBackdoor

\begin{theorem}[Random direction is an adversarial watermark (Theorem \ref{thm:random_stamp})]
\label{app:random_stamp}
\thmRandomStamp
\end{theorem}
\pfRandomStamp

\begin{theorem}[Theorem \ref{thm:scale_capacity}]
\label{app:scale_capacity}
\thmScaleCapacity
\end{theorem}
\pfScaleCapacity

\begin{lemma}[Lemma \ref{lemma:mem_vs_vc}]
\label{app:mem_vs_vc}
\lmMemVsVC
\end{lemma}
\pfMemVsVC

\begin{theorem}[Theorem \ref{thm:mcap_attack}]
\label{app:mcap_attack}
\thmMcapAttack
\end{theorem}
\begin{theorem}[Generalization of Theorem \ref{thm:mcap_attack}]
\label{app:mcap_attack_general}
\thmMcapAttackGeneral
\end{theorem}
\pfMcapAttack

\begin{theorem}[Theorem \ref{thm:mcap_zero}]
\label{app:mcap_zero}
\thmMcapZero
\end{theorem}
\pfMcapZero

\begin{example}[Overparameterized Linear Classifiers (Example \ref{ex:linear_backdoor_mcap})]
\label{app:linear_backdoor_mcap}
\exLinearBackdoorMcap
\end{example}
\pfLinearBackdoorMcap

\begin{example}[Linear Classifiers Over Convex Bodies (Example \ref{ex:linear_backdoor_cvx})]
\label{app:linear_backdoor_cvx}
\exLinearBackdoorCvx
\end{example}
\pfLinearBackdoorCvx

\begin{example}[Sign Changes (Example \ref{ex:sign_changes})]
\label{app:sign_changes}
\exSignChanges
\end{example}
\pfSignChanges

\begin{theorem}[Theorem \ref{thm:certify_backdoor}]
\label{app:certify_backdoor}
\thmCertifyBackdoor
\end{theorem}
\pfCertifyBackdoor

\begin{theorem}[Filtering Implies Generalization (Theorem \ref{thm:filter_to_generalize})]
\label{app:filter_to_generalize}
\thmFilterToGeneralize
\end{theorem}
\pfFilterToGeneralize

\begin{theorem}[Generalization Implies Filtering (Theorem \ref{thm:generalize_to_filter})]
\label{app:generalize_to_filter}
\thmGeneralizeToFilter
\end{theorem}
\pfGeneralizeToFilter

\newpage

\section{Numerical Trials}
\label{sec:numerical_trials}

In this section, we present a practical use case for Theorem \ref{thm:certify_backdoor} (Appendix Theorem \ref{app:certify_backdoor}).

Recall that, at a high level, Theorem \ref{thm:certify_backdoor} states that under certain assumptions, minimizing robust loss on the corrupted training set will either:
\begin{enumerate}
    \item Result in a low robust loss, which will imply from uniform convergence that the resulting classifier is robust to adversarial (and therefore backdoor) perturbations.
    \item Result in a high robust loss, which will be noticeable at training time.
\end{enumerate}
This suggests that practitioners can use adversarial training on a training set which may be backdoored and use the resulting robust loss value to make a decision about whether to deploy the classifier. To empirically validate this approach, we run this procedure (i.e., some variant of Algorithm \ref{alg:certify_backdoor}) on the MNIST handwritten digit classification task\footnote{We select MNIST because one can achieve a reasonably robust classifier on the clean version of the dataset. This helps us underscore the difference between the robust loss at train time with and without backdoors in the training set. Moreover, this allows us to explore a setting where our assumptions in Theorem \ref{thm:certify_backdoor} might not hold -- in particular, it's not clear that we have enough data to attain uniform convergence for the binary loss and $\robustloss$, and it's not clear how to efficiently minimize $\robustloss$.}(see \cite{lecun-mnisthandwrittendigit-2010}). Here, the learner wishes to recover a neural network robust to small $\ell_\infty$ perturbations and where the adversary is allowed to make a small $\ell_\infty$-norm watermark.

\paragraph{Disclaimers} As far as we are aware, the MNIST dataset does not contain personally identifiable information or objectionable content. The MNIST dataset is made available under the terms of the Creative Commons Attribution-Share Alike 3.0 License.

\paragraph{Reproducibility} We have included all the code to generate these results in the supplementary material. Our code can be found at \url{https://github.com/narenmanoj/mnist-adv-training}.\footnote{Some of our code is derived from the GitHub repositories \url{https://github.com/MadryLab/backdoor_data_poisoning} and \url{https://github.com/skmda37/Adversarial_Machine_Learning_Tensorflow}.}. Our code is tested and working with TensorFlow 2.4.1, CUDA 11.0, NVIDIA RTX 2080Ti, and the Google Colab GPU runtime.

\subsection{MNIST Using Neural Networks}
\label{subs:exp_mnist}

\subsubsection{Scenario}

Recall that the MNIST dataset consists of $10$ classes, where each corresponds to a handwritten digit in $\inbraces{0, \dots, 9}$. The classification task here is to recover a classifier that, upon receiving an image of a handwritten digit, correctly identifies which digit is present in the image.

In our example use case, an adversary picks a target label $t \in \inbraces{0, \dots, 9}$ and a small additive watermark. If the true classifier is $\hstar(x)$, then the adversary wants the learner to find a classifier $\hhat$ maximizing $\prvv{x\sim\cD | \hstar(x) \neq t}{\hhat(x) = t}$. In other words, this can be seen as a ``many-to-one'' attack, where the adversary is corrupting examples whose labels are not $t$ in order to induce a classification of $t$. The adversary is allowed to inject some number of examples into the training set such that the resulting fraction of corrupted examples in the training set is at most $\alpha$.

We will experimentally demonstrate that the learner can use the intuition behind Theorem \ref{thm:certify_backdoor} (Appendix Theorem \ref{app:certify_backdoor}) to either recover a reasonably robust classifier or detect the presence of significant corruptions in the training set. Specifically, the learner can optimize a proxy for the robust loss via adversarial training using $\ell_\infty$ bounded adversarial examples, as done by \cite{Madry2017-ep}.

\paragraph{Instantiation of Relevant Problem Parameters} Let $\cH$ be the set of neural networks with architecture as shown in Table \ref{table:architecture}. Let $\cX$ be the set of images of handwritten digits; we represent these as vectors in $\insquare{0,1}^{784}$. Define $\fadv$ below:
$$\inbraces{\patch{x} \suchthat \norm{x - \patch{x}}_\infty \le 0.3 \text{ and } \patch{x} - x = \mathsf{pattern}}$$
where $\mathsf{pattern}$ is the shape of the backdoor (we use an ``X'' shape in the top left corner of the image, inspired by \cite{Tran2018-bf}). We let the maximum $\ell_\infty$ perturbation be at most $0.3$ since this parameter has been historically used in training and evaluating robust networks on MNIST (see \cite{Madry2017-ep}). In our setup, we demonstrate that these parameters suffice to yield a successful backdoor attack on a vanilla training procedure (described in greater detail in a subsequent paragraph).

Although it is not clear how to efficiently exactly calculate and minimize $\robustloss$, we will approximate $\robustloss$ by calculating $\ell_\infty$-perturbed adversarial examples using a Projected Gradient Descent (PGD) attack. To minimize $\robustloss$, we use adversarial training as described in \cite{Madry2017-ep}. Generating Table \ref{table:mnist_full} takes roughly 155 minutes using our implementation of this procedure with TensorFlow 2.4.1 running on the GPU runtime freely available via Google Colab. We list all our relevant optimization and other experimental parameters in Table \ref{table:exp_params}.

\begin{table}[h]
\caption{Neural network architecture used in experiments. We implemented this architecture using the Keras API of TensorFlow 2.4.1.\label{table:architecture}}
\centering
\begin{tabular}{|l|l|}
\hline
\multicolumn{1}{|c|}{\textbf{Layer}}   & \multicolumn{1}{c|}{\textbf{Parameters}}                                                                                       \\ \hline
\texttt{Conv2D}       & \texttt{filters=32}, \texttt{kernel\_size=(3,3)},\texttt{activation='relu'} \\
\texttt{MaxPooling2D} & \texttt{pool\_size=(2,2)}                                                                                     \\
\texttt{Conv2D}       & \texttt{filters=64},\texttt{kernel\_size=(3,3)},\texttt{activation='relu'}  \\
\texttt{Flatten}      &                                                                                                                                \\
\texttt{Dense}        & \texttt{units=1024},\texttt{activation='relu'}                                               \\
\texttt{Dense}        & \texttt{units=10},\texttt{activation='softmax'}                                              \\ \hline
\end{tabular}
\end{table}

\begin{table}[h]
\caption{Experimental hyperparameters. We made no effort to optimize these hyperparameters; indeed, many of these are simply the default arguments for the respective TensorFlow functions.\label{table:exp_params}}
\centering
\begin{tabular}{|l|l|}
\hline
\multicolumn{1}{|c|}{\textbf{Property}} & \multicolumn{1}{c|}{\textbf{Details}}                                       \\ \hline
Epochs                                  & 2                                                                           \\
Validation Split                        & None                                                                        \\
Batch Size                              & 32                                                                          \\
Loss                                    & Sparse Categorical Cross Entropy                                            \\
Optimizer                               & RMSProp (step size = $0.001$, $\rho$ = 0.9, momentum = 0, $\eps = 10^{-7}$) \\
NumPy Random Seed                       & 4321                                                                        \\
TensorFlow Random Seed                  & 1234                                                                        \\
PGD Attack                              & $\eps = 0.3$, step size = $0.01$, iterations = $40$, restarts = $10$        \\ \hline
\end{tabular}
\end{table}

\paragraph{Optimization Details} See Table \ref{table:exp_params} for all relevant hyperparameters and see Table \ref{table:architecture} for the architecture we use.

For the ``Vanilla Training'' procedure, we use no adversarial training and simply use our optimizer to minimize our loss directly. For the ``PGD-Adversarial Training'' procedure, we use adversarial training with a PGD adversary.

In our implementation of adversarial training, we compute adversarial examples for each image in each batch using the PGD attack and we minimize our surrogate loss on this new batch. This is sufficient to attain a classifier with estimated robust loss of around $0.08$ on an uncorrupted training set. 

\subsubsection{Goals and Evaluation Methods}

We want to observe the impact of adding backdoor examples and the impact of running adversarial training on varied values of $\alpha$ (the fraction of the training set that is corrupted). 

To do so, we fix a value for $\alpha$ and a target label $t$ and inject enough backdoor examples such that exactly an $\alpha$ fraction of the resulting training set contains corrupted examples. Then, we evaluate the train and test robust losses on the training set with and without adversarial training to highlight the difference in robust loss observable to the learner. As sanity checks, we also include binary losses and test set metrics. For the full set of metrics we collect, see Table \ref{table:mnist_full}.

To avoid out-of-memory issues when computing the robust loss on the full training set (roughly $60000$ training examples and their adversarial examples), we sample $5000$ training set examples uniformly at random from the full training set and compute the robust loss on these examples. By Hoeffding's Inequality (see \cite{Vershynin2018-xn}), this means that with probability $0.99$ over the choice of the subsampled training set, the difference between our reported statistic and its population value is at most $\sim 0.02$.

\subsubsection{Results and Discussion}

\begin{table}[H]
\caption{Results with MNIST with a target label $t = 0$ and backdoor pattern ``X.'' In each cell, the top number represents the respective value when the network was trained without any kind of robust training, and the bottom number represents the respective value when the network was trained using adversarial training as per \cite{Madry2017-ep}. For example, at $\alpha = 0.05$, for Vanilla Training, the training $0-1$ loss is only $0.01$, but the training robust loss is $1.00$, whereas for PGD-Adversarial Training, the training $0-1$ loss is $0.07$ and the training robust loss is $0.13$. The Backdoor Success Rate is our estimate of $\prvv{x \sim \cD || y \neq t}{\patch{x} = t}$, which may be less than the value of the robust loss.\label{table:mnist_full}}
\centering
\begin{tabular}{|c|c|l|l|l|l|l|}
\hline
\multicolumn{2}{|c|}{$\alpha$}                                    & 0.00 & 0.05 & 0.15 & 0.20 & 0.30 \\ \hline
\multirow{2}{*}{Training $0-1$ Loss}   & Vanilla Training         & 0.01 & 0.01 & 0.01 & 0.01 & 0.01 \\
                                       & PGD-Adversarial Training & 0.02 & 0.07 & 0.17 & 0.22 & 0.33 \\ \hline
\multirow{2}{*}{Training Robust Loss}  & Vanilla Training         & 1.00 & 1.00 & 1.00 & 1.00 & 1.00 \\
                                       & PGD-Adversarial Training & 0.09 & 0.13 & 0.24 & 0.27 & 0.41 \\ \hline
\multirow{2}{*}{Testing $0-1$ Loss}    & Vanilla Training         & 0.01 & 0.01 & 0.01 & 0.02 & 0.01 \\
                                       & PGD-Adversarial Training & 0.02 & 0.03 & 0.03 & 0.03 & 0.06 \\ \hline
\multirow{2}{*}{Testing Robust Loss}   & Vanilla Training         & 1.00 & 1.00 & 1.00 & 1.00 & 1.00 \\
                                       & PGD-Adversarial Training & 0.09 & 0.09 & 0.11 & 0.10 & 0.19 \\ \hline
\multirow{2}{*}{Backdoor Success Rate} & Vanilla Training         & 0.00 & 1.00 & 1.00 & 1.00 & 1.00 \\
                                       & PGD-Adversarial Training & 0.00 & 0.00 & 0.01 & 0.00 & 0.05 \\ \hline
\end{tabular}
\end{table}

See Table \ref{table:mnist_full} for sample results from our trials. Over runs of the same experiment with varied target labels $t$, we attain similar results; see Section \ref{subs:full_tables} for the full results. We now discuss the key takeaways from this numerical trial.

\paragraph{Training Robust Loss Increases With $\alpha$} Observe that our proxy for $\robustloss(\hhat, S)$ increases as $\alpha$ increases. This is consistent with the intuition from Theorem \ref{thm:certify_backdoor} in that a highly corrupted training set is unlikely to have low robust loss. Hence, if the learner expects a reasonably low robust loss and fails to observe this during training, then the learner can reject the training set, particularly at high $\alpha$.

\paragraph{Smaller $\alpha$ and Adversarial Training Defeats Backdoor} On the other hand, notice that at smaller values of $\alpha$ (particularly $\alpha \le 0.20$), the learner can still recover a classifier with minimal decrease in robust accuracy. Furthermore, there is not an appreciable decrease in natural accuracy either when using adversarial training on a minimally corrupted training set. Interestingly, even at large $\alpha$, the test-time robust loss and binary losses are not too high when adversarial training was used. Furthermore, the test-time robust loss attained at $\alpha > 0$ is certainly better than that obtained when adversarial training is not used, even at $\alpha = 0$. Hence, although the practitioner cannot certify that the learned model is robust without a clean validation set, the learned model does tend to be fairly robust.

\paragraph{Backdoor Is Successful With Vanilla Training} Finally, as a sanity check, notice that when we use vanilla training, the backdoor trigger induces a targeted misclassification very reliably, even at $\alpha = 0.05$. Furthermore, the training and testing error on clean data is very low, which indicates that the learner would have failed to detect the fact that the model had been corrupted had they checked only the training and testing errors before deployment.

\paragraph{Prior Empirical Work} The work of \cite{Borgnia2021-vk} empirically shows the power of data augmentation in defending against backdoored training sets. Although their implementation of data augmentation is different from ours\footnote{Observe that our implementation of adversarial training can be seen as a form of adaptive data augmentation.}, their work still demonstrates that attempting to minimize some proxy for the robust loss can lead to a classifier robust to backdoors at test time. However, our evaluation also demonstrates that classifiers trained using adversarial training can be robust against test-time adversarial attacks, in addition to being robust to train-time backdoor attacks. Furthermore, our empirical results indicate that the train-time robust loss can serve as a good indicator for whether a significant number of backdoors are in the training set.

\newpage

\subsubsection{Results For All Target Labels}
\label{subs:full_tables}

Here, we present tables of the form of Table \ref{table:mnist_full} for all choices of target label $t \in \inbraces{0, \dots, 9}$. Notice that the key takeaways remain the same across all target labels.

\begin{table}[H]
\caption{Results with MNIST with a target label $t = 0$ and backdoor pattern ``X.''}
\centering
\begin{tabular}{|c|c|l|l|l|l|l|}
\hline
\multicolumn{2}{|c|}{$\alpha$}                                    & 0.00 & 0.05 & 0.15 & 0.20 & 0.30 \\ \hline
\multirow{2}{*}{Training $0-1$ Loss}   & Vanilla Training         & 0.01 & 0.01 & 0.01 & 0.01 & 0.01 \\
                                       & PGD-Adversarial Training & 0.02 & 0.07 & 0.17 & 0.22 & 0.33 \\ \hline
\multirow{2}{*}{Training Robust Loss}  & Vanilla Training         & 1.00 & 1.00 & 1.00 & 1.00 & 1.00 \\
                                       & PGD-Adversarial Training & 0.09 & 0.13 & 0.24 & 0.27 & 0.41 \\ \hline
\multirow{2}{*}{Testing $0-1$ Loss}    & Vanilla Training         & 0.01 & 0.01 & 0.01 & 0.02 & 0.01 \\
                                       & PGD-Adversarial Training & 0.02 & 0.03 & 0.03 & 0.03 & 0.06 \\ \hline
\multirow{2}{*}{Testing Robust Loss}   & Vanilla Training         & 1.00 & 1.00 & 1.00 & 1.00 & 1.00 \\
                                       & PGD-Adversarial Training & 0.09 & 0.09 & 0.11 & 0.10 & 0.19 \\ \hline
\multirow{2}{*}{Backdoor Success Rate} & Vanilla Training         & 0.00 & 1.00 & 1.00 & 1.00 & 1.00 \\
                                       & PGD-Adversarial Training & 0.00 & 0.00 & 0.01 & 0.00 & 0.05 \\ \hline
\end{tabular}
\end{table}

\begin{table}[H]
\caption{Results with MNIST with a target label $t = 1$ and backdoor pattern ``X.''}
\centering
\begin{tabular}{|c|c|l|l|l|l|l|}
\hline
\multicolumn{2}{|c|}{$\alpha$}                                    & 0.00 & 0.05 & 0.15 & 0.20 & 0.30 \\ \hline
\multirow{2}{*}{Training $0-1$ Loss}   & Vanilla Training         & 0.01 & 0.01 & 0.01 & 0.01 & 0.01 \\
                                       & PGD-Adversarial Training & 0.02 & 0.07 & 0.17 & 0.23 & 0.32 \\ \hline
\multirow{2}{*}{Training Robust Loss}  & Vanilla Training         & 1.00 & 1.00 & 1.00 & 1.00 & 1.00 \\
                                       & PGD-Adversarial Training & 0.08 & 0.12 & 0.23 & 0.32 & 0.38 \\ \hline
\multirow{2}{*}{Testing $0-1$ Loss}    & Vanilla Training         & 0.01 & 0.01 & 0.01 & 0.01 & 0.01 \\
                                       & PGD-Adversarial Training & 0.02 & 0.02 & 0.03 & 0.04 & 0.05 \\ \hline
\multirow{2}{*}{Testing Robust Loss}   & Vanilla Training         & 1.00 & 1.00 & 1.00 & 1.00 & 1.00 \\
                                       & PGD-Adversarial Training & 0.09 & 0.08 & 0.11 & 0.13 & 0.14 \\ \hline
\multirow{2}{*}{Backdoor Success Rate} & Vanilla Training         & 0.00 & 1.00 & 1.00 & 1.00 & 1.00 \\
                                       & PGD-Adversarial Training & 0.00 & 0.00 & 0.00 & 0.02 & 0.03 \\ \hline
\end{tabular}
\end{table}

\begin{table}[H]
\caption{Results with MNIST with a target label $t = 2$ and backdoor pattern ``X.''}
\centering
\begin{tabular}{|c|c|l|l|l|l|l|}
\hline
\multicolumn{2}{|c|}{$\alpha$}                                    & 0.00 & 0.05 & 0.15 & 0.20 & 0.30 \\ \hline
\multirow{2}{*}{Training $0-1$ Loss}   & Vanilla Training         & 0.01 & 0.01 & 0.01 & 0.01 & 0.00 \\
                                       & PGD-Adversarial Training & 0.02 & 0.07 & 0.17 & 0.22 & 0.32 \\ \hline
\multirow{2}{*}{Training Robust Loss}  & Vanilla Training         & 1.00 & 1.00 & 1.00 & 1.00 & 1.00 \\
                                       & PGD-Adversarial Training & 0.08 & 0.13 & 0.23 & 0.28 & 0.38 \\ \hline
\multirow{2}{*}{Testing $0-1$ Loss}    & Vanilla Training         & 0.01 & 0.02 & 0.01 & 0.02 & 0.01 \\
                                       & PGD-Adversarial Training & 0.02 & 0.03 & 0.03 & 0.03 & 0.05 \\ \hline
\multirow{2}{*}{Testing Robust Loss}   & Vanilla Training         & 1.00 & 1.00 & 1.00 & 1.00 & 1.00 \\
                                       & PGD-Adversarial Training & 0.09 & 0.09 & 0.10 & 0.10 & 0.14 \\ \hline
\multirow{2}{*}{Backdoor Success Rate} & Vanilla Training         & 0.00 & 1.00 & 1.00 & 1.00 & 1.00 \\
                                       & PGD-Adversarial Training & 0.00 & 0.00 & 0.00 & 0.01 & 0.04 \\ \hline
\end{tabular}
\end{table}

\begin{table}[H]
\caption{Results with MNIST with a target label $t = 3$ and backdoor pattern ``X.''}
\centering
\begin{tabular}{|c|c|l|l|l|l|l|}
\hline
\multicolumn{2}{|c|}{$\alpha$}                                    & 0.00 & 0.05 & 0.15 & 0.20 & 0.30 \\ \hline
\multirow{2}{*}{Training $0-1$ Loss}   & Vanilla Training         & 0.01 & 0.01 & 0.01 & 0.01 & 0.01 \\
                                       & PGD-Adversarial Training & 0.02 & 0.07 & 0.18 & 0.23 & 0.32 \\ \hline
\multirow{2}{*}{Training Robust Loss}  & Vanilla Training         & 1.00 & 1.00 & 1.00 & 1.00 & 1.00 \\
                                       & PGD-Adversarial Training & 0.08 & 0.13 & 0.23 & 0.28 & 0.38 \\ \hline
\multirow{2}{*}{Testing $0-1$ Loss}    & Vanilla Training         & 0.01 & 0.01 & 0.01 & 0.02 & 0.02 \\
                                       & PGD-Adversarial Training & 0.02 & 0.02 & 0.03 & 0.04 & 0.05 \\ \hline
\multirow{2}{*}{Testing Robust Loss}   & Vanilla Training         & 1.00 & 1.00 & 1.00 & 1.00 & 1.00 \\
                                       & PGD-Adversarial Training & 0.09 & 0.09 & 0.11 & 0.11 & 0.13 \\ \hline
\multirow{2}{*}{Backdoor Success Rate} & Vanilla Training         & 0.00 & 1.00 & 1.00 & 1.00 & 1.00 \\
                                       & PGD-Adversarial Training & 0.00 & 0.01 & 0.00 & 0.01 & 0.03 \\ \hline
\end{tabular}
\end{table}

\begin{table}[H]
\caption{Results with MNIST with a target label $t = 4$ and backdoor pattern ``X.''}
\centering
\begin{tabular}{|c|c|l|l|l|l|l|}
\hline
\multicolumn{2}{|c|}{$\alpha$}                                    & 0.00 & 0.05 & 0.15 & 0.20 & 0.30 \\ \hline
\multirow{2}{*}{Training $0-1$ Loss}   & Vanilla Training         & 0.01 & 0.01 & 0.01 & 0.01 & 0.01 \\
                                       & PGD-Adversarial Training & 0.02 & 0.07 & 0.17 & 0.22 & 0.32 \\ \hline
\multirow{2}{*}{Training Robust Loss}  & Vanilla Training         & 1.00 & 1.00 & 1.00 & 1.00 & 1.00 \\
                                       & PGD-Adversarial Training & 0.08 & 0.13 & 0.24 & 0.27 & 0.42 \\ \hline
\multirow{2}{*}{Testing $0-1$ Loss}    & Vanilla Training         & 0.01 & 0.01 & 0.01 & 0.01 & 0.01 \\
                                       & PGD-Adversarial Training & 0.02 & 0.02 & 0.03 & 0.03 & 0.05 \\ \hline
\multirow{2}{*}{Testing Robust Loss}   & Vanilla Training         & 1.00 & 1.00 & 1.00 & 1.00 & 1.00 \\
                                       & PGD-Adversarial Training & 0.08 & 0.09 & 0.11 & 0.10 & 0.15 \\ \hline
\multirow{2}{*}{Backdoor Success Rate} & Vanilla Training         & 0.00 & 1.00 & 1.00 & 1.00 & 1.00 \\
                                       & PGD-Adversarial Training & 0.00 & 0.00 & 0.01 & 0.01 & 0.04 \\ \hline
\end{tabular}
\end{table}

\begin{table}[H]
\caption{Results with MNIST with a target label $t = 5$ and backdoor pattern ``X.''}
\centering
\begin{tabular}{|c|c|l|l|l|l|l|}
\hline
\multicolumn{2}{|c|}{$\alpha$}                                    & 0.00 & 0.05 & 0.15 & 0.20 & 0.30 \\ \hline
\multirow{2}{*}{Training $0-1$ Loss}   & Vanilla Training         & 0.01 & 0.01 & 0.01 & 0.01 & 0.01 \\
                                       & PGD-Adversarial Training & 0.02 & 0.07 & 0.17 & 0.22 & 0.33 \\ \hline
\multirow{2}{*}{Training Robust Loss}  & Vanilla Training         & 1.00 & 1.00 & 1.00 & 1.00 & 1.00 \\
                                       & PGD-Adversarial Training & 0.07 & 0.13 & 0.23 & 0.28 & 0.41 \\ \hline
\multirow{2}{*}{Testing $0-1$ Loss}    & Vanilla Training         & 0.01 & 0.01 & 0.01 & 0.02 & 0.02 \\
                                       & PGD-Adversarial Training & 0.02 & 0.03 & 0.03 & 0.03 & 0.06 \\ \hline
\multirow{2}{*}{Testing Robust Loss}   & Vanilla Training         & 1.00 & 1.00 & 1.00 & 1.00 & 1.00 \\
                                       & PGD-Adversarial Training & 0.08 & 0.09 & 0.11 & 0.10 & 0.16 \\ \hline
\multirow{2}{*}{Backdoor Success Rate} & Vanilla Training         & 0.00 & 1.00 & 1.00 & 1.00 & 1.00 \\
                                       & PGD-Adversarial Training & 0.00 & 0.00 & 0.01 & 0.01 & 0.05 \\ \hline
\end{tabular}
\end{table}

\begin{table}[H]
\caption{Results with MNIST with a target label $t = 6$ and backdoor pattern ``X.''}
\centering
\begin{tabular}{|c|c|l|l|l|l|l|}
\hline
\multicolumn{2}{|c|}{$\alpha$}                                    & 0.00 & 0.05 & 0.15 & 0.20 & 0.30 \\ \hline
\multirow{2}{*}{Training $0-1$ Loss}   & Vanilla Training         & 0.01 & 0.01 & 0.01 & 0.01 & 0.01 \\
                                       & PGD-Adversarial Training & 0.02 & 0.07 & 0.17 & 0.22 & 0.33 \\ \hline
\multirow{2}{*}{Training Robust Loss}  & Vanilla Training         & 1.00 & 1.00 & 1.00 & 1.00 & 1.00 \\
                                       & PGD-Adversarial Training & 0.08 & 0.12 & 0.24 & 0.27 & 0.40 \\ \hline
\multirow{2}{*}{Testing $0-1$ Loss}    & Vanilla Training         & 0.01 & 0.02 & 0.01 & 0.01 & 0.01 \\
                                       & PGD-Adversarial Training & 0.02 & 0.03 & 0.03 & 0.03 & 0.06 \\ \hline
\multirow{2}{*}{Testing Robust Loss}   & Vanilla Training         & 1.00 & 1.00 & 1.00 & 1.00 & 1.00 \\
                                       & PGD-Adversarial Training & 0.09 & 0.09 & 0.12 & 0.10 & 0.16 \\ \hline
\multirow{2}{*}{Backdoor Success Rate} & Vanilla Training         & 0.00 & 1.00 & 1.00 & 1.00 & 1.00 \\
                                       & PGD-Adversarial Training & 0.00 & 0.00 & 0.01 & 0.01 & 0.04 \\ \hline
\end{tabular}
\end{table}

\begin{table}[H]
\caption{Results with MNIST with a target label $t = 7$ and backdoor pattern ``X.''}
\centering
\begin{tabular}{|c|c|l|l|l|l|l|}
\hline
\multicolumn{2}{|c|}{$\alpha$}                                    & 0.00 & 0.05 & 0.15 & 0.20 & 0.30 \\ \hline
\multirow{2}{*}{Training $0-1$ Loss}   & Vanilla Training         & 0.01 & 0.01 & 0.01 & 0.01 & 0.01 \\
                                       & PGD-Adversarial Training & 0.02 & 0.07 & 0.18 & 0.22 & 0.32 \\ \hline
\multirow{2}{*}{Training Robust Loss}  & Vanilla Training         & 1.00 & 1.00 & 1.00 & 1.00 & 1.00 \\
                                       & PGD-Adversarial Training & 0.07 & 0.12 & 0.25 & 0.29 & 0.39 \\ \hline
\multirow{2}{*}{Testing $0-1$ Loss}    & Vanilla Training         & 0.01 & 0.01 & 0.01 & 0.02 & 0.01 \\
                                       & PGD-Adversarial Training & 0.02 & 0.03 & 0.03 & 0.03 & 0.04 \\ \hline
\multirow{2}{*}{Testing Robust Loss}   & Vanilla Training         & 1.00 & 1.00 & 1.00 & 1.00 & 1.00 \\
                                       & PGD-Adversarial Training & 0.08 & 0.08 & 0.11 & 0.10 & 0.13 \\ \hline
\multirow{2}{*}{Backdoor Success Rate} & Vanilla Training         & 0.00 & 1.00 & 1.00 & 1.00 & 1.00 \\
                                       & PGD-Adversarial Training & 0.00 & 0.00 & 0.00 & 0.00 & 0.03 \\ \hline
\end{tabular}
\end{table}

\begin{table}[H]
\caption{Results with MNIST with a target label $t = 8$ and backdoor pattern ``X.''}
\centering
\begin{tabular}{|c|c|l|l|l|l|l|}
\hline
\multicolumn{2}{|c|}{$\alpha$}                                    & 0.00 & 0.05 & 0.15 & 0.20 & 0.30 \\ \hline
\multirow{2}{*}{Training $0-1$ Loss}   & Vanilla Training         & 0.01 & 0.01 & 0.01 & 0.01 & 0.01 \\
                                       & PGD-Adversarial Training & 0.02 & 0.07 & 0.17 & 0.22 & 0.32 \\ \hline
\multirow{2}{*}{Training Robust Loss}  & Vanilla Training         & 1.00 & 1.00 & 1.00 & 1.00 & 1.00 \\
                                       & PGD-Adversarial Training & 0.08 & 0.14 & 0.23 & 0.28 & 0.41 \\ \hline
\multirow{2}{*}{Testing $0-1$ Loss}    & Vanilla Training         & 0.01 & 0.01 & 0.01 & 0.01 & 0.01 \\
                                       & PGD-Adversarial Training & 0.02 & 0.03 & 0.03 & 0.03 & 0.05 \\ \hline
\multirow{2}{*}{Testing Robust Loss}   & Vanilla Training         & 1.00 & 1.00 & 1.00 & 1.00 & 1.00 \\
                                       & PGD-Adversarial Training & 0.08 & 0.09 & 0.11 & 0.10 & 0.17 \\ \hline
\multirow{2}{*}{Backdoor Success Rate} & Vanilla Training         & 0.00 & 1.00 & 1.00 & 1.00 & 1.00 \\
                                       & PGD-Adversarial Training & 0.00 & 0.00 & 0.01 & 0.01 & 0.05 \\ \hline
\end{tabular}
\end{table}

\begin{table}[H]
\caption{Results with MNIST with a target label $t = 9$ and backdoor pattern ``X.''}
\centering
\begin{tabular}{|c|c|l|l|l|l|l|}
\hline
\multicolumn{2}{|c|}{$\alpha$}                                    & 0.00 & 0.05 & 0.15 & 0.20 & 0.30 \\ \hline
\multirow{2}{*}{Training $0-1$ Loss}   & Vanilla Training         & 0.01 & 0.01 & 0.01 & 0.01 & 0.01 \\
                                       & PGD-Adversarial Training & 0.02 & 0.07 & 0.17 & 0.22 & 0.33 \\ \hline
\multirow{2}{*}{Training Robust Loss}  & Vanilla Training         & 1.00 & 1.00 & 1.00 & 1.00 & 1.00 \\
                                       & PGD-Adversarial Training & 0.08 & 0.13 & 0.23 & 0.29 & 0.43 \\ \hline
\multirow{2}{*}{Testing $0-1$ Loss}    & Vanilla Training         & 0.01 & 0.01 & 0.01 & 0.01 & 0.01 \\
                                       & PGD-Adversarial Training & 0.02 & 0.03 & 0.03 & 0.04 & 0.06 \\ \hline
\multirow{2}{*}{Testing Robust Loss}   & Vanilla Training         & 1.00 & 1.00 & 1.00 & 1.00 & 1.00 \\
                                       & PGD-Adversarial Training & 0.09 & 0.10 & 0.11 & 0.11 & 0.20 \\ \hline
\multirow{2}{*}{Backdoor Success Rate} & Vanilla Training         & 0.00 & 1.00 & 1.00 & 1.00 & 1.00 \\
                                       & PGD-Adversarial Training & 0.01 & 0.01 & 0.01 & 0.01 & 0.06 \\ \hline
\end{tabular}
\end{table}

\end{document}